\documentclass[12pt]{article}

\usepackage[margin=1.1in]{geometry}
\usepackage{macros}

\title{Three Types of Calibration with Properties and their Semantic and Formal Relationships}
\author{Rabanus Derr, Jessie Finocchiaro, Robert C. Williamson}

\begin{document}

\maketitle

\begin{abstract}
    Fueled by discussions around ``trustworthiness'' and algorithmic fairness, calibration of predictive systems has regained scholars attention.
    The vanilla definition and understanding of calibration is, simply put, on all days on which the rain probability has been predicted to be $p$, the actual frequency of rain days was $p$.
    However, the increased attention has led to an immense variety of new notions of ``calibration.''
    Some of the notions are incomparable, serve different purposes, or imply each other.
    In this work, we provide two accounts which motivate calibration: self-realization of forecasted properties and precise estimation of incurred losses of the decision makers relying on forecasts.
    We substantiate the former via the reflection principle and the latter by actuarial fairness.
    For both accounts we formulate prototypical definitions via properties $\Gamma$ of outcome distributions, e.g., the mean or median.
    The prototypical definition for self-realization, which we call \emph{$\Gamma$-calibration}, is equivalent to a certain type of swap regret under certain conditions. These implications are strongly connected to the omniprediction learning paradigm.
    The prototypical definition for precise loss estimation is a modification of \emph{decision calibration} adopted from \citet{zhao_calibrating_2021}.
    For binary outcome sets both prototypical definitions coincide under appropriate choices of reference properties.
    For higher-dimensional outcome sets, both prototypical definitions can be subsumed by a natural extension of the binary definition, called \emph{distribution calibration} with respect to a property.
    We conclude by commenting on the role of groupings in both accounts of calibration often used to obtain multicalibration. 
    In sum, this work provides a semantic map of calibration in order to navigate a fragmented terrain of notions and definitions.
\end{abstract}

\section{Introduction}
Calibration has increasingly gained interest since it seems to provide a mathematical criterion of ``trustworthy''\footnote{We hesitate to provide a concise definition of trustworthiness in this work for the ambiguity and vagueness of its purpose and meaning. Instead, we write ``trustworthiness'' to emphasize to the reader that the term itself is and should be loaded with more than the formal intuitions for some facets of trustworthiness presented in this work.} predictions \citep{noarov2024calibration} and it is a major component of studies on algorithmic fairness \citep{chouldechova2017fair, hebert2018multicalibration}. Furthermore, the advent of capable, deep learning techniques gave rise to investigations of calibration of general deep neural networks \citep{guo_calibration_2017} and large language models \citep{cruz_evaluating_2024,openai_gpt-4_2024,kalai_calibrated_2024}.

While calibration has become a quantity of interest in empirical studies \citep{openai_gpt-4_2024,perdomo_difficult_2023,cruz_evaluating_2024}, theoretical works came up with dozens of new (and old) definitions of calibration investigating particular relationships between them. To name a few: $\Gamma$-calibration \citep{noarov_statistical_2023}, decision calibration \citep{zhao_calibrating_2021}, $u$-calibration \citep{kleinberg_u-calibration_2023}, class-wise calibration \citep{kull2019beyond}, confidence calibration \citep{guo_calibration_2017}, Global Interpretable Calibration Index \citep{cabitza2022global}, distance to calibration \citep{blasiok_unifying_2023}, smooth calibration \citep{foster2018smooth}.
The ``jungle'' of current notions of calibration is difficult to navigate.

As a result, wide-spread use of the term ``calibration'' has blurred the boundaries of what is meant by it. 
However, all usages have in common that calibration is either a criterion to judge predictions in light of obtained data or the process of fulfilling such a criterion. 
In this work we stick to the former (\cf \citep{holtgen2023richness}).
Our goal is \emph{not} to provide an exhaustive list of notions of calibration.\footnote{After trying this for some time, we gave up on this journey due to the immense variety and subtleties of the suggested variants of calibration.} 
Instead, we follow our main question: \textbf{What is the abstract purpose of calibration?} This way we provide a semantic map and guide through key notions of calibration and their central relationships.

Calibration is often contrasted with pure predictive accuracy \citep{cruz_evaluating_2024, van2019calibration, seidenfeld1985calibration}. 
While predictive accuracy guarantees the ``usefulness'' of predictions, calibration guarantees the ``trustworthiness'' of the predictions.
The ``usefulness'' narrative is readily justified. Expected risk minimization, the core principle behind a large portion of learning techniques, is the negative analogue of expected utility maximization. The lower the risk, the more useful the predictions are, when measured with the corresponding loss (respectively utility) function.

The ``trustworthiness''-narrative, however, requires a more detailed explanation. 
We identified two accounts of calibration which could justify it: (a) \emph{Self-realization:} For the instances where some value $c$ was forecasted, the actual outcomes can be summarized to a value close to $c$. (b) \emph{Precise Loss Estimation:} The forecasted values let one provide estimates of incurred losses (for certain loss functions) which are close to the actual materialized losses.
While these accounts are not an exhaustive list, these two paradigms account for the majority of calibration motivations. 
We found a third account of calibration in the literature which focuses on the approximate equivalence of means. Here, predictions are calibrated if the average of outcomes is equal to the average of predictions on a certain subgroups or even individuals \citep{dawid1985calibration, zhao2020individual, luo2022local,holtgen2023richness}. This account is somewhat located between the narratives of ``usefulness'' and ``trustworthiness.''
Note that all accounts extend on a certain facet of the binary vanilla calibration definition (Definition~\ref{def: binary vanilla calibration}) which has historically motivated previous definitions. 

To uncover the two central accounts of this work, we rewrite current definitions of calibration using the language of properties \citep{osband1985providing,fissler_higher_2016,lambert_eliciting_2008}.
Properties, simply put, are functions from distributions on the outcome set $\Y$ to some prediction set $\R$, e.g., $\R = \reals$ or $\R = \Delta(\Y)$.
We distinguish between \emph{optimization-level} and \emph{decision-level} properties. 
Optimization-level properties specify the actual predicted entity, e.g., full distribution, mean, $\alpha$-quantile, or best action.
Decision-level properties most often correspond to downstream uses of the optimization-level property, such as decision makers informing their discrete action via a prediction about outcome probabilities, or ranking of classes being used to deduce the top-$k$ most likely outcomes. 
Decision-level properties borrow their name since they generalize maximum expected utility decision makers.
Notably, decisions are often discrete, and therefore hard to optimize directly, necessitating the distinction between (often continuous) optimization-level properties and (often discrete) decision-level properties.
Properties subsume not only utility-based decision makers, but rather arbitrary statistical properties of distributions as well, such as quantiles\footnote{Calibrated quantiles are relevant for conformal predictions \citep{jung_batch_2022}.}, ratios of expectations, or class marginals, to name a few. Furthermore, the elicitability and identifiability of properties as detailed in Section~\ref{Playground} relate properties to loss functions and their optimization criterion. This makes properties a useful vehicle to study calibration in the abstract.

Our contributions are summarized as follows:
\begin{enumerate}
    \item We subsume many existing definitions under three core ``types'' of definitions of calibration: distribution calibration with respect to a (decision-level) property $\Phi$ (Definition~\ref{def: Distribution Calibration with Respect to Gamma}), property calibration with respect to an abstract property $\Gamma$ (Definition~\ref{def: Gamma-calibration}) and decision calibration with respect to a set of loss functions $\L$ (Definition~\ref{def: decision calibration}).
    \item We show that all ``types'' of definitions of calibration collapse in the case of binary outcome sets and appropriate choices of $\Phi, \Gamma$ and $\L$ (Proposition~\ref{prop:Decision Calibration Equivalent to Mean Calibration for Binary Outcome Set}). More generally, we argue that distribution calibration is the central ``parent'' notion which implies both of the others (Proposition~\ref{prop:Distribution Calibration with Respect to Gamma implies Gamma-Calibration} and Proposition~\ref{prop:Distribution Calibration Implies Decision Calibration}), as summarized in Figure~\ref{fig: calibration backbone}. When generalizing to the approximate case, see Propositions~\ref{prop:approximate distribution calibration implies gamma calibration} and~\ref{prop:Approximate Distribution Calibration Implies Approximate Decision Calibration} and Figure~\ref{fig: calibration backbone approximate}. %
    \item We elaborate on property calibration as a prototypical definition for the account of self-realization. 
    We show it can equivalently be expressed as a type of swap regret (Proposition~\ref{prop:Perfect Calibration via Loss Function or Identification Function} and Proposition~\ref{prop:Approximate Calibration and Low Swap Regret}). Furthermore, we prove that self-realization is inherited by \emph{refined} properties--- those properties $\Phi$ that can be defined by composing a function with some original property $\Gamma$ (Proposition~\ref{prop:Calibration is Inherited to Refined Properties} and Proposition~\ref{prop:Approximate Calibration is Inherited to Refined Properties}), a characteristic of the notion which is exploited in omniprediction (\cf Proposition~\ref{prop:Low Swap Regret Guarantee for Refined Properties}).
    \item We contextualize decision calibration as a prototypical notion of precise loss estimation. In particular, we show it requires simultaneous precise Bayes risk estimation for several loss functions of interest (Proposition~\ref{prop:Decision Calibration Implies Precise Bayes Risk Estimation}). Furthermore, we dissect the relationship between self-realization and precise loss estimation, concluding that these accounts are incommensurable (Section~\ref{sec:When Self-Realization Implies Precise Loss Estimation, and Vice-Versa}).
    \item We provide a non-exhaustive categorization of existing notions of calibration in the three ``types'' in Table~\ref{tab:notions of calibration - strict derivatives} and Table~\ref{tab:notions of calibration - philosophy}.
\end{enumerate}

The distinction of the three ``types'' of definitions of calibration and their relationship to the presented accounts of self-realization and precise loss estimation has, to the best of our knowledge, not been made in the literature. Properties have been used in \citep{gneiting_regression_2023} and \citep{noarov_statistical_2023} to formalize calibration. Some of the formal relationships have been discussed already in literature (e.g., distribution calibration implies decision calibration has been argued for in \citep{zhao_calibrating_2021}), but not within the more general picture of properties. %

Figure~\ref{fig: calibration backbone} graphically summarizes the two strands which we will follow in this work. The figure depicts the idealized ``perfect'' calibration case. For approximate notions the implications require additional assumptions (Figure~\ref{fig: calibration backbone approximate}). Clearly, we don't expect the reader to yet make sense of the definitions and implications.
\begin{figure}[ht]
    \centering
    \def\svgwidth{0.99\columnwidth}
    {\footnotesize
    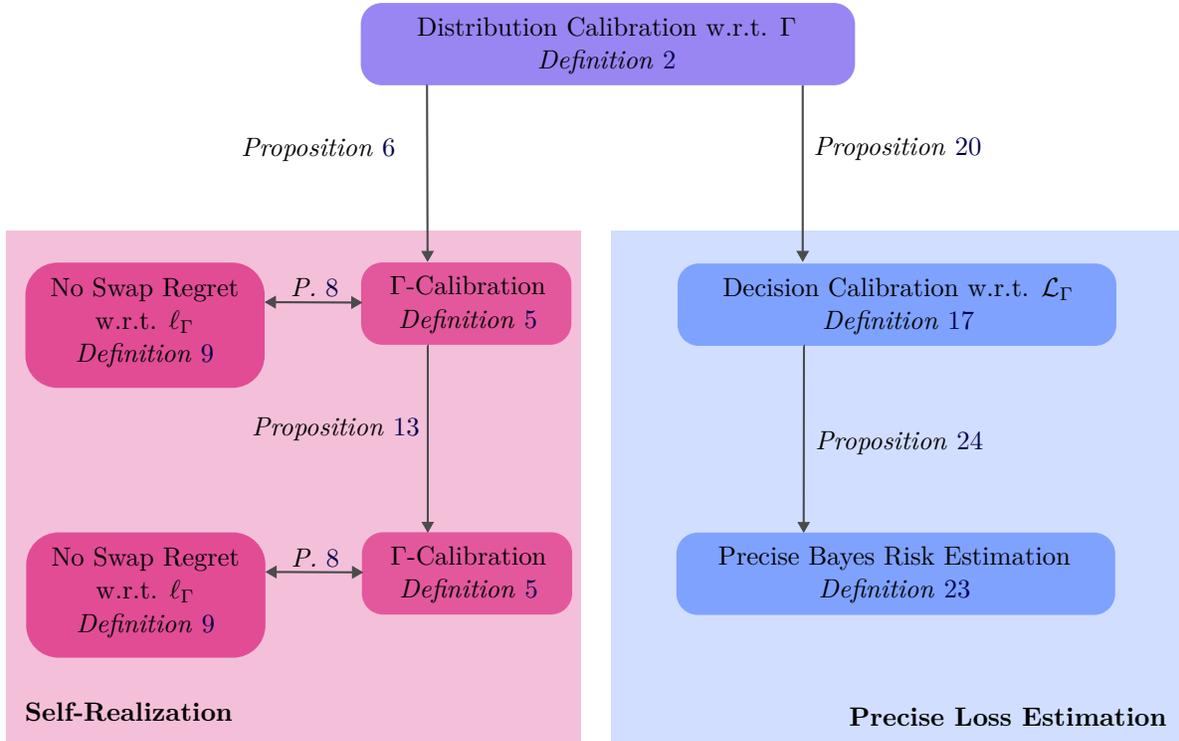
    }
    \caption{Relationships between Notions of Calibration. Implications under perfect calibration, finite $\Y$ and elicitable property $\Gamma$ and $\Phi$. The three types of calibration are marked in different colors. The abstract accounts of calibration are shaded.}
    \label{fig: calibration backbone}
\end{figure}
\begin{figure}[ht]
    \centering
    \def\svgwidth{0.99\columnwidth}
    {\footnotesize
    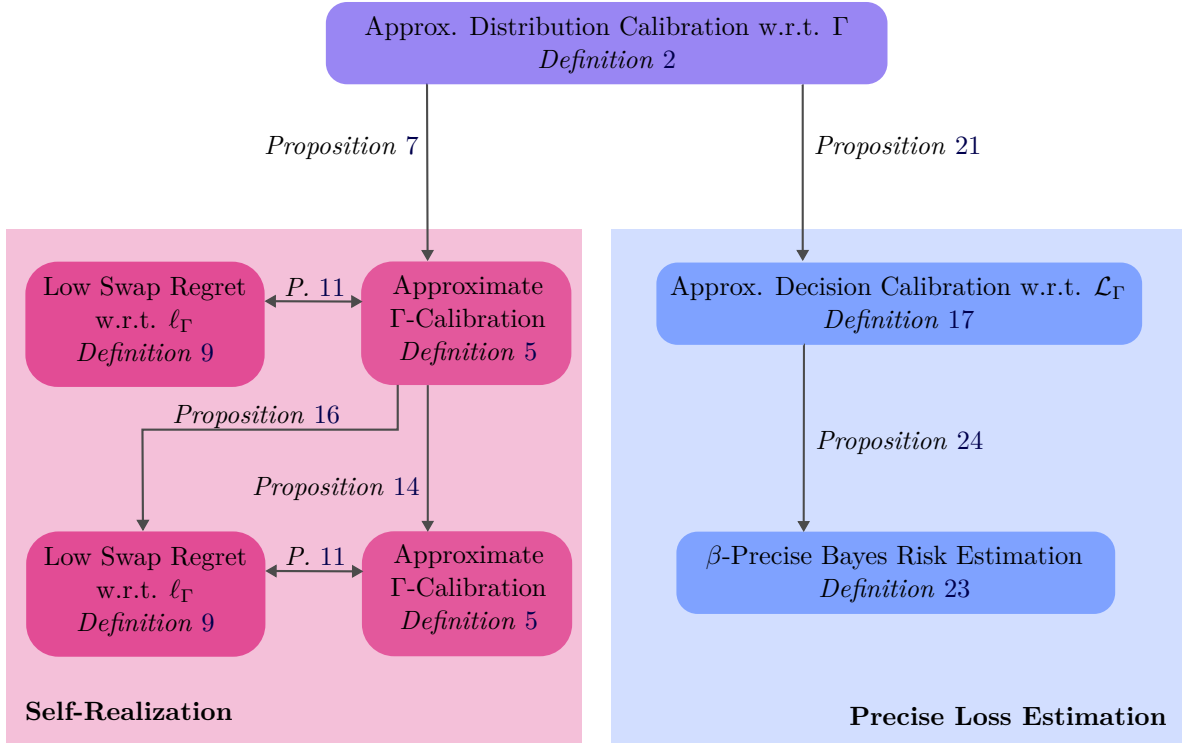
    }
    \caption{Relationship between Approximate Notions of Calibration. Implications Under Approximate Calibration and Finite $\Y$. Further conditions are stated in the the referenced propositions. Those conditions particularly contain Lipschitz and Smoothness assumptions. The three types of calibration are marked in different colors. The abstract accounts of calibration are shaded.}
    \label{fig: calibration backbone approximate}
\end{figure}

\section{An Exemplary Motivation}
\label{an exemplary motivation}
Let us walk through an example to get an intuitive understanding of the subtleties of different formalizations of calibration.

\citet{perdomo_difficult_2023} study the Dropout Early Warning System (DEWS) implemented at Wisconsin Public Schools. Based on state level testing in the 8th grade, the system predicted on-time high school graduation likelihood. Hence, the DEWS predictions themselves, i.e., the \emph{optimization-level property}, are probability forecasts on a binary outcome set. As the authors observe, the forecasts are not perfectly calibrated following \emph{binary vanilla calibration} (Definition~\ref{def: binary vanilla calibration}), i.e., the graduation rate of students who get a certain prediction $p$ is not $p$. 

The predictions of the DEWS is given to the department authorities which defines three risk categories ``high risk,'' ``medium risk,'' and ``low risk'' for dropout. The risk category determines the measures to be taken to increase the graduation likelihood. In other words, the risk categories form the \emph{decision-level property}. Are the \emph{predictions calibrated conditioned on the risk categories}, i.e., is the average graduation rate among all persons who are categorized as ``high risk'' (or ``medium risk'' or ``low risk'') equal to the average prediction score of those persons? This property, which we call \emph{distribution calibration with respect to the risk categories} (Definition~\ref{def: Distribution Calibration with Respect to Gamma}), is not fulfilled as immediately derived from \citep[Figures 1 and 2]{perdomo_difficult_2023}.

But, the average graduation rate among all persons who are categorized as ``high risk'' is lower than the average graduation rate in all other categories (for the other categories respectively). In other words, the ``high risk'' category self-realizes. Hence, the predictions are \emph{property calibrated} (Definition~\ref{def: Gamma-calibration}).

Finally, self-realization of the decision-level property is not necessarily the intention of calibration. Instead one can ask whether the department authorities can estimate their expected loss in mis-assigning students to the wrong group. For this one would need to introduce a loss function, whose minimizer is the risk category assignment. Note that several such loss functions exist. Depending on the choice of the loss function, the predictions are precisely estimating the loss, i.e., the predictions are \emph{decision calibrated} (Definition~\ref{def: decision calibration}), or not.

The detailed numbers of how calibrated the DEWS predictions are is irrelevant for the lessons we want to convey in this section. Some of the notions of calibration are fulfilled to a lesser degree than others. We use this example to demonstrate that there seems to be no ``right'' definition of calibration. This leaves a choice open which type and definition of calibration to consider in a contextualized problem. This choice has semantics and implications which we disentangle in the following.

\section{Formal Setup}
\label{Playground}

\paragraph{Data Distribution}
Let $(\Omega, \Sigma, \lambda)$ be a standard probability space. Let $\X$ (and $\Y$) be an input set (and an outcome set respectively).
For the sake of simplicity, we assume that those sets are finite dimensional Euclidean spaces or finite sets, equipped with the standard Borel-$\sigma$-algebra $\B(\X)$ (respectively $\B(\Y)$). We define two random variables $X \colon \Omega \rightarrow \X$ and $Y \colon \Omega \rightarrow \Y$. The distribution induced on $\X \times \Y$ by the composed random variable $(X,Y)$ is denoted $D$, and called the \emph{data distribution}. The set of all data distributions is denoted $\Delta(\X \times \Y)$.
For the expected value of a measurable function $g \colon \X \times \Y \rightarrow \reals$ we write $\mathbb{E}_{(X,Y) \sim D}[g(X,Y)]$ or simply $\mathbb{E}_{D}[g(X,Y)]$.

The marginals are denoted $D_X$ (respectively $D_Y$).
Whenever we use the notation $\P$ we mean a subset of $\Delta(\Y)$ the set of all probability distributions on the measurable space $(\Y, \B(\Y))$.
We call a data distribution $D$ \emph{regular with respect to $\P$} if and only if for every $x \in \X$ it is true that $D_{Y|  X} \in \P$.

Let $\kappa_{Y|X}\colon \X \times \B(\Y) \rightarrow [0,1]$ be a regular conditional probability distribution defined through $D$. In particular, we write $D_{Y|\{ x \colon X=x\}}$ for the probability measure $\kappa_{Y | X}(x, \cdot)$ with any fixed $x \in \X$. Note this regular conditional distribution exists (since $(\Y, \B(\Y))$ is a Borel space) \citep[Theorem 8.37]{klenke2013probability}. Let $f \colon \X \rightarrow\reals$ be a measurable function. An often used conditional distribution for a fixed value $\gamma \in \reals$ is, for all $A \in \B(\Y)$,
\begin{align*}
    D_{Y | f(X)=\gamma}(A) \coloneqq \frac{1}{\int_\X \llbracket x \colon f(x) =\gamma \rrbracket dD_X(x) } \int_\X \llbracket x \colon f(x) =\gamma \rrbracket \kappa_{Y|X}(x, A) dD_X(x).
\end{align*}
(\citep[Definition 8.28]{klenke2013probability}).
For the corresponding conditional expectation for a measurable function $g \colon \Y \rightarrow \reals$ we write,
\begin{align*}
    \mathbb{E}_{D}\left[ g(Y) |f(X) = \gamma \right] \coloneqq \mathbb{E}_{Y \sim D_{Y | f(X)=\gamma}}\left[ g(Y) \right].
\end{align*}

\paragraph{Properties}
\label{paragraph:properties}
Let $\R$ be a set of property values equipped with a metric $m$, i.e., $(\R,m)$ forms a metric space. Generally, we see three common choices of value sets: (i) if $\R \subseteq \reals$ is an interval, we set $m$ to be the absolute difference, (ii) if $|\R| < \infty$, we set $m$ to be the discrete metric. Finally, (iii) if $\R = \Delta(\Y)$, we set $m$ to be the total variation distance. Furthermore, $(\R, \mathcal{B}(\R))$ is a measurable space for $\mathcal{B}(\R)$ being the Borel-$\sigma$-algebra induced by the topology given via $m$.
A measurable function $\Gamma \colon \P \rightarrow \R$ is called a \emph{property}.\footnote{Note that measurability can be guaranteed by inducing a $\sigma$-algebra on $\P$ or simply referring to the Borel-$\sigma$-algebra on $\Delta(\Y)$ induced by the total variation distance, in particular if $\Y$ is finite.} Without loss of generality we assume that $\R = \im \Gamma$.
To distinguish between optimization-level and decision-level properties we sometimes use $\Gamma$ and $\Phi$ respectively. In Section~\ref{sec:calibration-as-precise-loss-estimation}, we additionally use $\Theta$ to denote Bayes risk properties.

The property is \emph{continuous} iff $\Gamma$ is continuous with respect to the total variation distance on $\P$ and the metric $m$ on $\R$. In particular, a property is \emph{Lipschitz continuous} iff $\Gamma$ is Lipschitz continuous with respect to the total variation distance on $\P$ and the metric $m$ on $\R$.
    
A property $\Gamma \colon \P \rightarrow \R$ is called \emph{elicitable} if there exists a loss function $\ell \colon \Y \times \R \rightarrow \reals$ which is measurable in the first variable for all fixed $\gamma \in \R$ in the second variable and such that,
\begin{align*}
    \Gamma(P) \in \argmin_{\gamma \in \R} \mathbb{E}_{Y \sim P}[\ell(Y, \gamma)],
\end{align*}
for all $P \in \P$. In the case $\ell$ elicits $\Gamma$, we say the loss $\ell$ is $\P$-consistent with respect to $\Gamma$.
Overloading notation, we sometimes write $\ell(P, \gamma) = \mathbb{E}_{Y \sim P}[\ell(Y, \gamma)]$. An elicitable property can be understood as a ``best response'' for a minimum expected loss decision maker.

A property $\Gamma \colon \P \rightarrow \R$ is called \emph{identifiable} if there exists an identification function $V \colon \Y \times \R \rightarrow \reals$ which is measurable in both variables, and %
\begin{align*}
    \Gamma(P) = \gamma \Leftrightarrow \mathbb{E}_{Y\sim P}[V(Y, \gamma)] = 0,
\end{align*}
for all $P \in \P$. We overload notation and write $V(P, \gamma) = \mathbb{E}_{Y \sim P}[V(Y, \gamma)]$.

A measurable function $f \colon \X \to \R$ is called a \emph{$\Gamma$-predictor} if $\im f \subseteq \R$ for the property $\Gamma : \P \to \R$.
Properties, as well as property predictors can be simply stacked to vectors, e.g., the distribution predictor $f \colon \X \to \Delta(\Y)$ is a $(\Gamma_y)_{y \in \Y}$-predictor, where $\Gamma_y$ denotes the property: probability of outcome $y \in \Y$.

\section{Distribution Calibration}
Calibration has been well-studied long before its evolution as a ``trustworthiness'' criterion or ``fairness'' criterion (as multi-calibration) in machine learning \citep{murphy_verification_1967, murphy_reliability_1977, dawid_well-calibrated_1982,degroot_comparison_1983}.\footnote{In some of the older literature the terms ``reliability'' and ``validity'' were (inconsistently) used instead of ``calibration''.}\footnote{ \citet{murphy_reliability_1977} provide the first calibration plots, objective frequency versus prediction.} Calibration was considered part of a canon of ``goodness''-measures of (meteorological) forecasts \citep{murphy_verification_1967, degroot_comparison_1983}. In particular, calibration captured the following intuition: if it rains a $p$-proportion of times on the days on which the prediction is $p$ for rain, the predictions are calibrated. Definition~\ref{def: binary vanilla calibration} makes this formal.
\begin{definition}[Binary Vanilla Calibration]
\label{def: binary vanilla calibration}
    Let $\X \subseteq \mathbb{R} $ be an input set and $\Y = \{ 0,1\}$ a binary outcome set. A predictor $f \colon \X \rightarrow [0,1]$ predicting the mean, i.e., the probability $D_{Y|X=x}(Y = 1) = \mathbb{E}[Y|X=x]$, is \emph{calibrated} on a distribution $D \in \Delta(\X \times \Y)$ if and only if, for all $\gamma \in \im f$,
    \begin{align*} 
        \mathbb{E}_{(X,Y) \sim D}[Y | f(X) = \gamma] = \gamma.
    \end{align*}
\end{definition}
Consider $f\colon \X \rightarrow \P$ to be a full distribution forecast for a finite $\Y$. It is not always relevant, nor attainable to provide full distribution estimates which are calibrated conditioned on the distributional predictions. In particular, in high-dimensional class probability prediction problems it is not realistic to achieve reasonable calibration guarantees \citep{roth_forecasting_2024,Noarov_2024_letall}. The number of events to condition on grows exponential in the dimensionality of the outcome set $\Y$. Hence, scholarship suggested conditioning on ``decision-events'' \citep{noarov2023high, roth_forecasting_2024}, marginals \citep{kull2019beyond}, or top-labels \citep{guo_calibration_2017, gupta_top-label_2021}. In other words, the number of conditioning events is reduced by focusing only on the ``relevant'' ones.

We subsume all those notions via a specified property. For the sake of readable notation we use $\Gamma$ as the symbol for the property here, even though it can refer to a decision-level or optimization-level property. For instance, $\Gamma$ could be the map which maps all distributions to one of their marginals (cf. \citep{kull2019beyond}) or best responses with respect to some utility function (cf. \citep{noarov2023high}).
\begin{definition}[Distribution Calibration with Respect to $\Gamma$]
\label{def: Distribution Calibration with Respect to Gamma}
    Let $\Gamma \colon \P \rightarrow \R$ be a property and $D$ a regular data distribution on $\X \times \Y$, where $\Y$ is finite. Let $f \colon \X \to \P$ be a distributional predictor. The predictor $f$ is \emph{distribution calibrated with respect to $\Gamma$} on $D$ if for every $\gamma \in \im \Gamma \circ f$,
    \begin{align*}
        \mathbb{E}_{D}[\llbracket Y = y \rrbracket|\Gamma \circ f(X) = \gamma] = \mathbb{E}_{D}[f_y(X)|\Gamma \circ f(X) = \gamma], \quad \forall y \in \Y,
    \end{align*}
    where $f_y(x) \in [0,1]$ denotes the $y$-component of the prediction $f(x) \in \Delta(\Y)$ for $x \in \X$. The predictor $f$ is \emph{$\alpha(\gamma)$-approximate distribution calibrated with respect to $\Gamma$} on $D$ if for every $\gamma \in \im \Gamma \circ f$, $\alpha(\gamma) \in \reals$,
    \begin{align*}
        \left| \mathbb{E}_{D}[\llbracket Y = y \rrbracket|\Gamma \circ f(X) = \gamma] - \mathbb{E}_{D}[f_y(X)|\Gamma \circ f(X) = \gamma] \right|\le \alpha(\gamma), \quad \forall y \in \Y.
    \end{align*}
\end{definition}
We illustrate Definition~\ref{def: Distribution Calibration with Respect to Gamma} in Figure~\ref{fig: distribution calibration on simplex}, where we demonstrate exactly (left) and approximately (right) satisfying distribution calibration with respect to the property $\Gamma(p) = \argmax_y p_y$ representing the mode.
\begin{figure}
\centering
\begin{subfigure}{0.48\textwidth}
    \includegraphics[width=\textwidth]{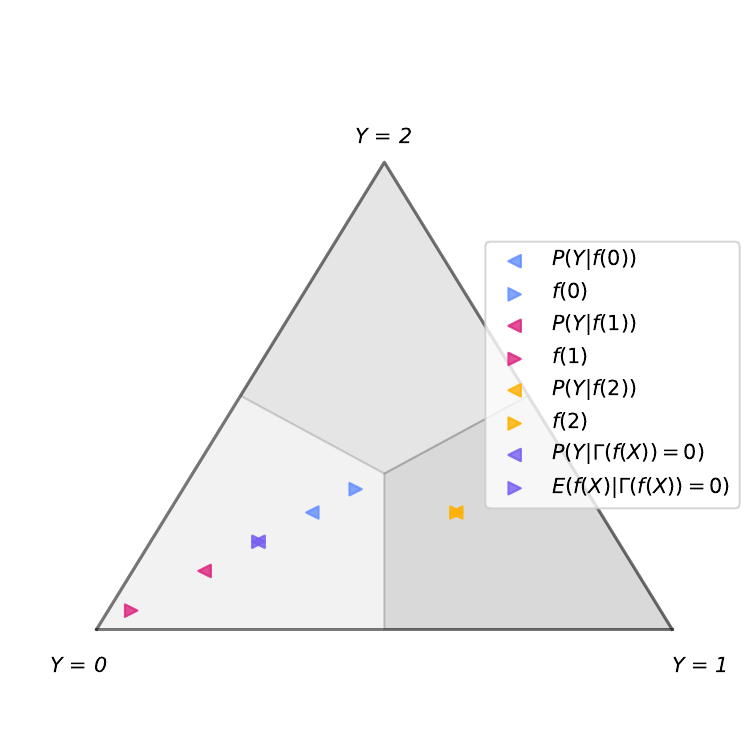}
    \caption{Perfect distribution calibration with respect to $\Gamma$.}
    \label{fig:distribution calibration perfect}
\end{subfigure}
\hfill
\begin{subfigure}{0.48\textwidth}
    \includegraphics[width=\textwidth]{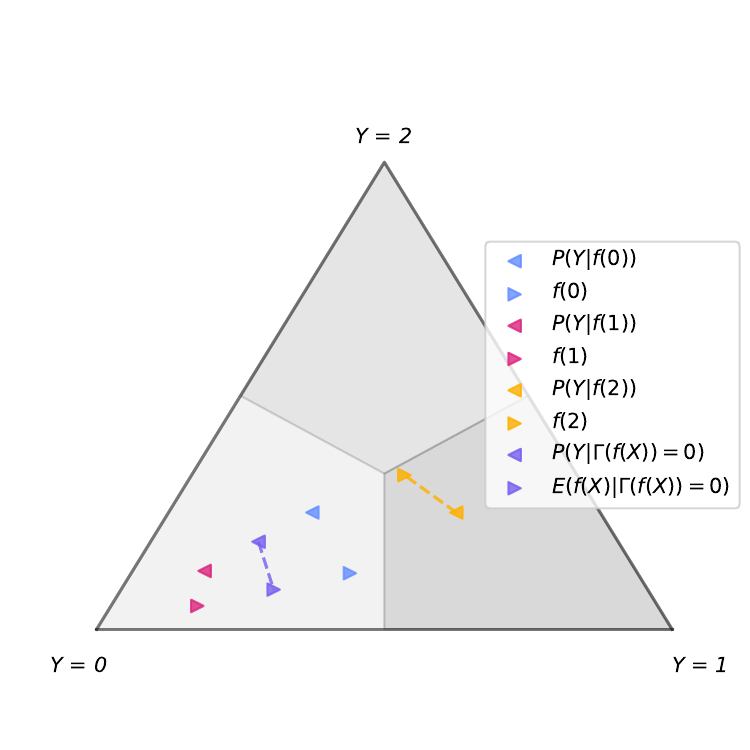}
    \caption{Approximate distribution calibration with respect to $\Gamma$.}
    \label{fig:distribution calibration notperfect}
\end{subfigure}
        
\caption{Illustration of distribution calibration. The outcome set is defined as $\Y = \{ 0,1,2\}$, the input set $\X = \{ 0,1,2\}$. We define $\Gamma(P) = \argmax_{y \in \Y} P(Y = y)$. The level sets of $\Gamma$ are drawn in different shades of gray. The left-directing markers denote the true conditional distribution for different choices of $x \in \X$. The right-directing markers denote the predicted distribution by a predictor $f$. The purple markers are convex combinations of the blue and red markers, where the convex combination is defined through the marginal distribution on $\X$ which is in our case fixed to be uniform. The dashed lines highlight the deviation from the true outcome distribution conditioned on a value of $\Gamma \circ f$ versus the expected forecast conditioned on a value of $\Gamma \circ f$. Only the forecasts change when comparing Figure~\ref{fig:distribution calibration perfect} versus Figure~\ref{fig:distribution calibration notperfect}.}
\label{fig: distribution calibration on simplex}
\end{figure}

Note that, in principle, $\Y$ does not need to be finite. However, it reduces technicalities. Furthermore, if $\Y$ is infinite, most predictions actually directly refer to properties and not the full distribution (e.g., in Gaussian Process Regression, the mean and covariance are estimated, which are properties of the full distribution).

\subsection{Distribution Calibration is Inherited}
Distribution calibration with respect to some property is naturally inherited by \emph{refined properties}. A property $\Phi$ is refined by another property $\Gamma$, if there exists a mapping which applied on $\Gamma$ gives $\Phi$.
\begin{definition}[Property Refinement~{\citep[Definition 12]{frongillo_elicitation_2021}}]
\label{def:property refinement}
    A property $\Phi \colon \P \to \R'$ is called \emph{refined by $\Gamma$}, if $\Gamma \colon \P \to \R$ is a property such that there exists a function $\phi \colon \R \to \R'$ with $\Phi = \phi \circ \Gamma$.\footnote{Property refinement is intricately related to the notion of \emph{indirect property elicitation} ~\citep{frongillo_elicitation_2021,finocchiaro_embedding_2024,finocchiaro_unifying_2021}}%
\end{definition}
\begin{proposition}[Distribution Calibration is Inherited]
\label{prop:Distribution Calibration is Inherited}
    Let $\Gamma \colon \P \to \R$ be a property and $D$ a regular data distribution on $\X \times \Y$, where $\Y$ is finite. Let $f \colon \X \to \P$ be a distributional predictor with $| \im f | < \infty$. If the predictor $f$ is $\alpha(\gamma)$-approximate distribution calibrated with respect $\Gamma$ on $D$, then it is $\alpha'(c)$-approximate distribution calibrated with respect to $\Phi \coloneqq \phi \circ \Gamma$ for all $\phi \colon \R \to \R'$ where $\alpha'(c) \coloneqq \sup_{\gamma \in \im f\colon \phi(\gamma) = c} \alpha(\gamma)$.
\end{proposition}
\begin{proof}
    By assumption,
    \begin{align*}
        \left| \mathbb{E}_{D}[\llbracket Y = y \rrbracket|\Gamma \circ f(X) = \gamma] - \mathbb{E}_{D}[f_y(X)|\Gamma \circ f(X) = \gamma] \right|\le \alpha(\gamma), \quad \forall y \in \Y.
    \end{align*}
    Hence, in particular, $\forall y \in \Y.$
    \begin{align*}
        &\alpha'(c) = \sup_{\gamma \in \im f\colon \phi(\gamma) = c} \alpha(\gamma)\\
        &\ge \sum_{\gamma \in \im f\colon \phi(\gamma) = c} P_D(f(X) = \gamma) \left| \mathbb{E}_{D}[\llbracket Y = y \rrbracket|\Gamma \circ f(X) = \gamma] - \mathbb{E}_{D}[f_y(X)|\Gamma \circ f(X) = \gamma] \right|\\
        &\ge  \left| \sum_{\gamma \in \im f\colon \phi(\gamma) = c} P_D(f(X) = \gamma) \left(\mathbb{E}_{D}[\llbracket Y = y \rrbracket|\Gamma \circ f(X) = \gamma] - \mathbb{E}_{D}[f_y(X)|\Gamma \circ f(X) = \gamma] \right) \right|\\
        &\ge  \left| \mathbb{E}_{D}[\llbracket Y = y \rrbracket|\Phi \circ f(X) = c] - \mathbb{E}_{D}[f_y(X)|\Phi \circ f(X) = c] \right|.
    \end{align*}
\end{proof}
Under mild geometric conditions on the predictions, perfect distribution calibration with respect to all elicitable binary properties implies distribution calibration (Appendix~\ref{app: dist calibration wrt to all binary properties}). This statement can be understood as a reverse implication to above Proposition~\ref{prop:Distribution Calibration is Inherited}.
Figure~\ref{fig: property refinement and inheritance of distribution calibration} illustrates the idea of property refinement and inheritance of distribution calibration. 
\begin{figure}[ht]
\centering
    \includegraphics[width=0.5\textwidth]{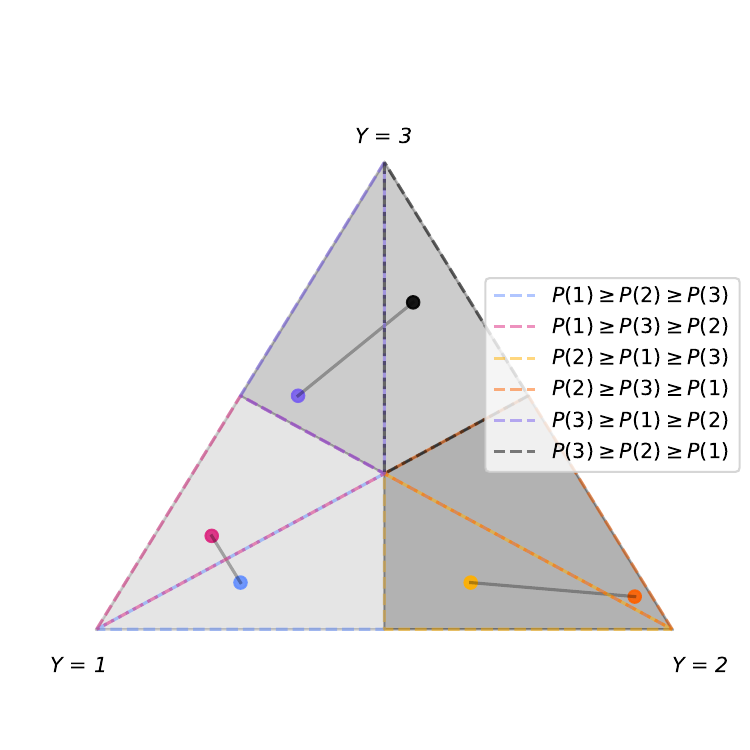}
    \caption{Illustration of property refinement and inheritance of distribution calibration. The outcome set is defined as $\Y = \{ 0,1,2\}$, the input set $|\X| = 6$. We define $\Gamma(P) = (y_1, y_2, y_3)$ such that $P(Y = y_1) \ge P(Y = y_2) \ge P(Y = y_3)$ and $\Phi(P) = \argmax_{y \in \Y} P(Y = y)$. The level sets of $\Gamma$ have colored boundaries listed in the legend. The level sets of $\Phi$ are colored respectively drawn in different shades of gray. The property $\Gamma$ refines $\Phi$.
    We assume that $f$ is a distributional predictor whose predictions are marked as dots. The color of the dots represents $\Gamma\circ f(x)$. The lines indicate the convex combination of predictions happening when conditioning on $\Phi \circ f(X)$ instead of $\Gamma \circ f(X)$. Since the level sets of $\Phi$ are all convex the line is always contained \emph{within} a single level set.}
    \label{fig: property refinement and inheritance of distribution calibration}
\end{figure}

Distribution calibration with respect to $\Gamma$ forms the entry point to the two accounts of calibration mentioned earlier (cf. Figure~\ref{fig: calibration backbone} and igure~\ref{fig: calibration backbone approximate}). Distribution calibration is close to the account of mean-matching briefly mentioned in the introduction. However, this account puts more focus on groups and individuals, i.e., information provided through the input $\X$ (\cf \citep{holtgen2023richness}).

\section{\texorpdfstring{$\Gamma$}{Gamma}-Calibration as Self-Realization of Predictions}
\label{Calibration as Self-realization of predictions}
Central to vanilla calibration is the self-referential aspect in the conditioning, i.e., given the predicted Bernoulli distribution has parameter $\gamma$, the actual distribution of outcomes is Bernoulli-distributed with parameter $\gamma$. This self-reference is linked to a widely discussed principle for deference of belief: the \emph{reflection principle} \citep{van1984belief} (Appendix~\ref{calibration and the reflection principle}). In this section we generalize calibration around this perspective.

As observed by \citep{jung_batch_2022,jung_moment_2021,gupta_online_2022} the expectation operator in Definition~\ref{def: binary vanilla calibration} can be replaced by moments, respectively quantile functions. From these works, \citet{noarov_statistical_2023} distilled a general definition of calibration with respect to properties. Different to \citet{noarov_statistical_2023} we ignore groupings based on input $X$. Hence, we consider $\Gamma$-calibration and not ``multi''-$\Gamma$-calibration.

\begin{definition}[$\Gamma$-Calibration \citep{noarov_statistical_2023}]
\label{def: Gamma-calibration}
    Suppose $(\R, m)$ is a metric space.
    Let $\Gamma \colon \P \rightarrow \R$ be a property and $D$ a regular data distribution. Let $f \colon \X \to \R$ be a $\Gamma$-predictor. The $\Gamma$-predictor $f$ is \emph{$\Gamma$-calibrated} on $D$ if for every $\gamma \in \im f$,
    \begin{align*}
        \Gamma(D_{Y | f(X) = \gamma}) = \gamma.
    \end{align*}
    The $\Gamma$-predictor $f$ is \emph{$\alpha(\gamma)$-approximate $\Gamma$-calibrated} on $D$ if for every $\gamma \in \im f$,
    \begin{align*}
        m(\Gamma(D_{Y | f(X) = \gamma}), \gamma) \le \alpha(\gamma).
    \end{align*}
\end{definition}
Central to the definition is the self-realization aspect. The property of the distribution of outcomes on which the predictor $f$ predicted $\gamma$ is $\gamma$. The reader familiar with \citep{noarov_statistical_2023} might question the use of a general property $\Gamma$ in Definition~\ref{def: Gamma-calibration}. For a discussion see Appendix~\ref{app: A note on sensibility to calibration}.

\citet{gneiting_regression_2023} introduced ``T-calibration'' (Definition 2.7 therein), which is a measure-theoretic definition of $\Gamma$-calibration. The authors already note that certain properties $\Gamma$ are not sensible to calibration following \citep{noarov_statistical_2023}, but only give necessary not sufficient conditions for the sensibility.

In practice it is rarely the case that a $\Gamma$-predictor is perfectly $\Gamma$-calibrated. However, there are algorithms, such as \citep[Algorithm 1]{noarov_statistical_2023} which post-hoc approximately calibrate $\Gamma$-predictors (for identifiable $\Gamma$).

Remember our conventions for $m$ listed in Section~\ref{paragraph:properties}, e.g., if $\R$ is an interval on the real line, then $m$ is set to be the absolute difference.\footnote{Note that metric $m$ does \emph{not} define a distance to calibration in the sense of \citep{blasiok_unifying_2023}.}
Note that in many existing definitions the authors commit to a specified aggregation of $\alpha(\gamma)$ over all $\gamma \in \im f$. For instance, \citep{noarov_statistical_2023} consider the expected squared value, $\alpha \coloneqq \mathbb{E}_{\gamma \sim D_{f(X)}}[\alpha(\gamma)^2]$. For a summary of such aggregations see \citep{garg_oracle_2023}. We illustrate Definition~\ref{def: Gamma-calibration} in Figure~\ref{fig: property calibration on simplex}.
\begin{figure}
\centering
\begin{subfigure}{0.4\textwidth}
    \includegraphics[width=\textwidth]{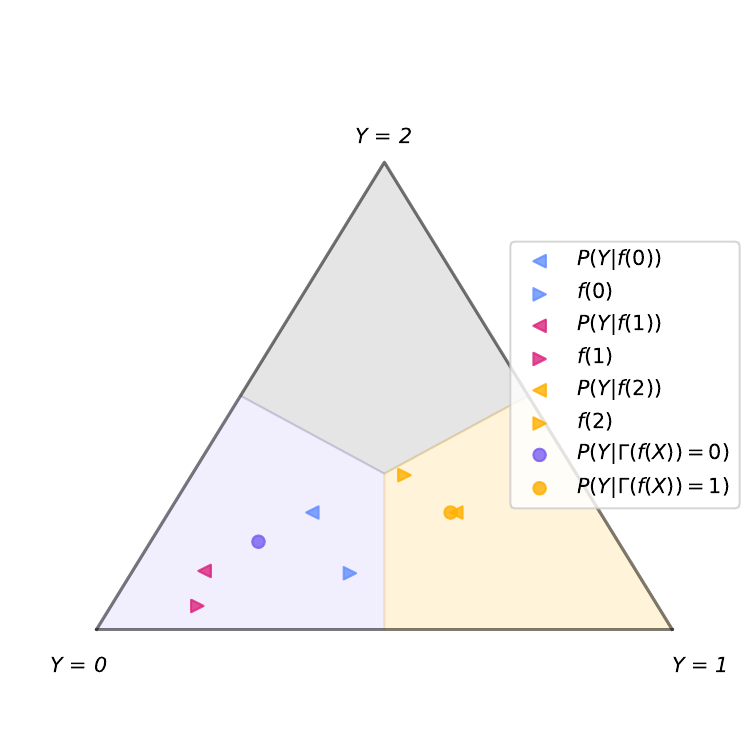}
    \caption{Perfect $\Gamma$-calibration.}
    \label{fig:property calibration perfect}
\end{subfigure}
\hfill
\begin{subfigure}{0.4\textwidth}
    \includegraphics[width=\textwidth]{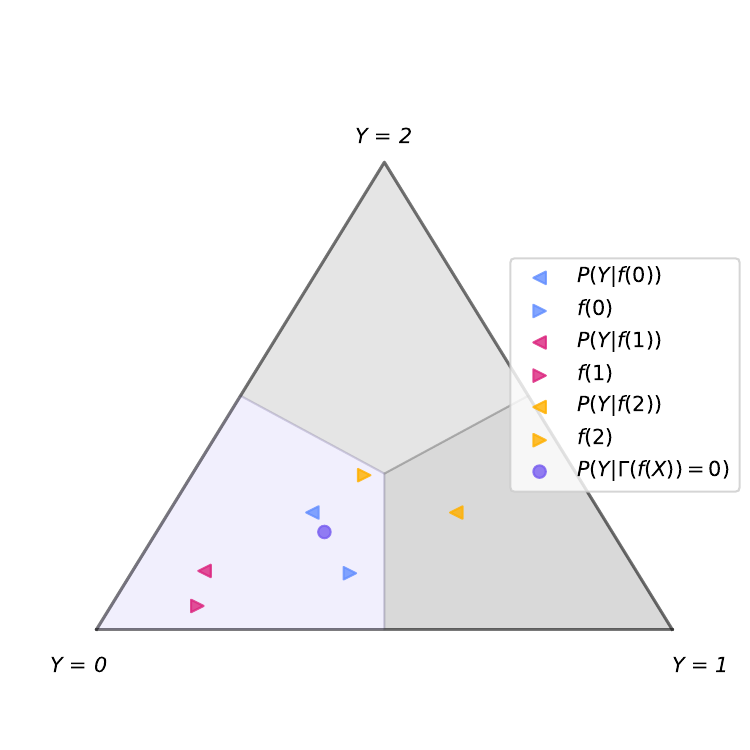}
    \caption{Perfect $\Gamma$-calibration II.}
    \label{fig:property calibration perfect 2}
\end{subfigure}
\hfill
\begin{subfigure}{0.4\textwidth}
    \includegraphics[width=\textwidth]{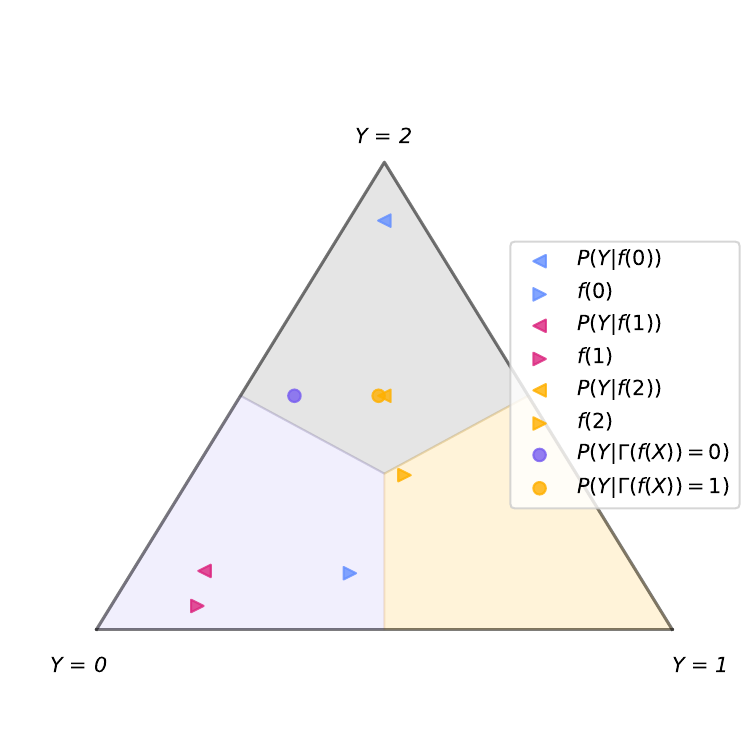}
    \caption{Approximate $\Gamma$-calibration.}
    \label{fig:property calibration notperfect}
\end{subfigure}
        \caption{Illustration of $\Gamma$-calibration. The outcome set is defined as $\Y = \{ 0,1,2\}$, the input set $\X = \{ 0,1,2\}$. We define $\Gamma(P) = \argmax_{y \in \Y} P(Y = y)$. The level sets of $\Gamma$ are colored respectively drawn in different shades of gray. The left-directing markers denote the true conditional distribution for different choices of $x \in \X$. We assume that $f$ is a distributional predictor (right-directing markers) which is then fed into $\Gamma$. The dots denote the true outcome distribution conditioned on a value of $\Gamma \circ f$. If the dot is in the level set of the same color, then the prediction is perfectly $\Gamma$-calibrated. Note that Figure~\ref{fig:property calibration perfect} uses the same predictions and outcome distributions as Figure~\ref{fig:distribution calibration notperfect} showing that predictions could be perfect $\Gamma$-calibrated while only being approximately distribution calibrated with respect to $\Gamma$. Figure~\ref{fig:property calibration perfect 2} shows that $\Gamma$-calibration does no necessarily require the true conditional distributions to all live in the correct level set. Only the predictor $f$ is changed from Figure~\ref{fig:property calibration perfect} to Figure~\ref{fig:property calibration perfect 2}. Figure~\ref{fig:property calibration notperfect} then changes the true conditional distribution compared to Figure~\ref{fig:property calibration perfect} but fixes $f$, which lead to the violation of the perfect $\Gamma$-calibration constraint.}
\label{fig: property calibration on simplex}
\end{figure}

\subsection{Distribution Calibration Implies \texorpdfstring{$\Gamma$}{Gamma}-Calibration}
If $\Y$ is finite, then predictors realistically can provide full distribution estimates. In this case, distribution calibration with respect to a property $\Gamma$ almost immediately implies $\Gamma$-calibration. The argument seems more involved as it requires careful treatment of convex combination of distribution on which distributional predictions with equal property values have been made. For the sake of readability, we restrict Proposition~\ref{prop:Distribution Calibration with Respect to Gamma implies Gamma-Calibration} to finite-valued predictors, $|\im f | < \infty$, as it prevents us from using technical machinery as in \citep[Theorem 3.6]{noarov_statistical_2023}. Furthermore, this assumption is arguably a mild one which is fulfilled in every practical setting on empirical distributions.
\begin{proposition}[Distribution Calibration with Respect to $\Gamma$ implies $\Gamma$-Calibration]
    \label{prop:Distribution Calibration with Respect to Gamma implies Gamma-Calibration}
    Let $\Gamma \colon \P \rightarrow \R$ be a property with convex level sets and $D$ a regular data distribution on $\X \times \Y$, where $\Y$ is finite. Let $f \colon \X \to \P$ be a predictor which is distribution calibrated with respect to $\Gamma$. We assume that $|\im f| < \infty$. Then, $\Gamma \circ f$ is $\Gamma$-calibrated.
\end{proposition}
\begin{proof}
    We have to show that for all $\gamma \in \im \Gamma \circ f$,
    \begin{align}
    \label{eq:perfect gamma calibration in proof}
        \Gamma(D_{Y | \Gamma \circ f(X) = \gamma}) = \gamma.
    \end{align}
    Note that for all $y \in \Y$,
    \begin{align*}
        D_{Y | \Gamma \circ f(X) = \gamma}(Y = y) &= \mathbb{E}_{D}[\llbracket Y = y \rrbracket|\Gamma \circ f(X) = \gamma]\\
        &= \mathbb{E}_{D}[f_y(X)|\Gamma \circ f(X) = \gamma]\\
        &= \sum_{d \in \im f} D_{X}(f(X) = d| \Gamma\circ f(X) = \gamma) d_y,
    \end{align*}
    where $d_y \in [0,1]$ denotes the $y$-component of the distribution $d \in \Delta(\Y)$.
    It follows,
    \begin{align*}
        D_{Y | \Gamma \circ f(X) = \gamma} = \sum_{d \in \im f} D_{X}(f(X) = d| \Gamma\circ f(X) = \gamma) d.
    \end{align*}
    Since $\Gamma$ has convex level sets, the convex combination of $d \in \im f$ gives the desired Equation~\eqref{eq:perfect gamma calibration in proof}.
\end{proof}
The above Proposition~\ref{prop:Distribution Calibration with Respect to Gamma implies Gamma-Calibration} holds for arbitrary properties with convex level sets. Unfortunately, in the approximate case the property $\Gamma$ needs to be smooth. Discrete properties $\Gamma$ don't fulfill this condition, because a slight mismatch of the average predictions can lead to a catastrophic mismatch in the according $\Gamma$-values.
\begin{proposition}[Approximate Distribution Calibration w.r.t. $\Gamma$ implies approximate $\Gamma$-Calibration for smooth $\Gamma$]
\label{prop:approximate distribution calibration implies gamma calibration}
Let $\Gamma \colon \P \rightarrow \R$ with $\R \subseteq \reals$ be a $K$-Lipschitz property with convex level sets and $D$ a regular data distribution on $\X \times \Y$, where $\Y$ is finite. Let $f \colon \X \to \P$  be a distributional predictor that is $\alpha(\gamma)$-approximately distributionally calibrated with respect to $\Gamma$. Suppose that $|\im f| < \infty$. Then, $\Gamma \circ f$ is $|\Y|K\alpha(\gamma)$-approximately $\Gamma$-calibrated.
\end{proposition}
\begin{proof}
Observe that, since $\Gamma$ is real-valued, the metric $m$ is the absolute difference.
    We have to show that for all $\gamma \in \im \Gamma \circ f$,
    \begin{align}
    \label{eq:approximate gamma calibration in proof}
        | \Gamma(D_{Y | \Gamma \circ f(X) = \gamma}) - \gamma| \le |\Y|K\alpha(\gamma).
    \end{align}

    To this end, first observe that
    \begin{align*}
        \mathbb{E}_{D}[f(X)|\Gamma \circ f(X) = \gamma] = \sum_{d \in \im f} d \cdot D_{X}(f(X) = d| \Gamma\circ f(X) = \gamma) \quad \in \Delta(\Y),
    \end{align*}
    is a convex combination of predicted values $d$.
    Because $\Gamma(d) = \gamma$ for $d \in \im f$ such that $\Gamma\circ f(D) = \gamma$ and $\Gamma$ has convex level sets, $\Gamma(\mathbb{E}_{D}[f(X)|\Gamma \circ f(X) = \gamma]) = \gamma$.

    Second, the total variation distance between $D_{Y | \Gamma \circ f(X) = \gamma}$ and $\mathbb{E}_{D}[f(X)|\Gamma \circ f(X) = \gamma]$ is bounded above for all $\gamma \in \im f$,
    \begin{align*}
        \sup_{A \subseteq \Y} \left| \sum_{y \in A} D_{Y | \Gamma \circ f(X) = \gamma}(Y = y) -  \mathbb{E}_{D}[f_y(X)|\Gamma \circ f(X) = \gamma] \right| \le | \Y | \alpha(\gamma)
    \end{align*}
    where $f_y(X)$ denotes the $y$-component of the prediction $f(X) \in \Delta(\Y)$. This follows from the fact that $f$ is $\alpha(\gamma)$-approximate distribution calibrated with respect to $\Gamma$. Hence, by applying $\Gamma$ and exploiting the Lipschitzness of $\Gamma$, we obtain the desired Equation~\eqref{eq:approximate gamma calibration in proof}.
\end{proof}

\subsection{\texorpdfstring{$\Gamma$}{Gamma}-Calibration is Equivalent to Low Swap Regret}
 It is no surprise that property calibration can be equivalently reformulated for loss (respectively identification) functions, if the property is elicitable (respectively identifiable).\footnote{Surprisingly, already in \citep{a_verification_2016} the terms ``calibration'' and ``identification'' got linked to each other. However, their Definition 4.2 of calibration is largely disconnected from the use of the term in current machine learning literature.}
\begin{proposition}[Perfect $\Gamma$-Calibration via Loss Function or Identification Function]
\label{prop:Perfect Calibration via Loss Function or Identification Function}
    Let $\Gamma \colon \P \rightarrow \R$ be a property and $D$ a regular data distribution. Let $f$ be $\Gamma$-calibrated on a regular data distribution $D$.
    \begin{enumerate}[(a)]
        \item If $\Gamma$ is elicitable with $\ell$, then $f$ is $\Gamma$-calibrated on $D$ if and only if, for all $\gamma \in \im f$,
        \begin{align*}
             \mathbb{E}_{D}\left[ \ell(Y, \gamma) | f(X) = \gamma \right] -  \min_{\hat{\gamma} \in \im \Gamma} \mathbb{E}_{D}\left[ \ell(Y, \hat{\gamma}) | f(X) = \gamma \right] = 0.
        \end{align*}
        \item If $\Gamma$ is identifiable with $V$, then $f$ is $\Gamma$-calibrated on $D$ if and only if, for all $\gamma \in \im f$,
        \begin{align*}
            \mathbb{E}_{D}\left[ V(Y, \gamma)|f(X) = \gamma\right] = 0.
        \end{align*}
    \end{enumerate}
\end{proposition}
\begin{proof}
    By definition of elicitability and identifiability.
\end{proof}
Statement (a) is essentially a ``no swap regret'' statement. It is different from precise loss estimation as in decision calibration (Section~\ref{sec:calibration-as-precise-loss-estimation}). We now argue that this relationship extends to approximate calibration. To this end, we formally introduce a distributional definition of swap regret similar in spirit to \citep[p. 91]{cesa2006prediction}.
\begin{definition}[Swap Regret]
\label{def:swap regret}
    Let $\Gamma \colon \P \to \R$ be an elicitable property with $\P$-consistent scoring function $\ell \colon \Y \times \R \rightarrow \reals$ and $D$ a regular data distribution. Let $f \colon \X \to \R$ be a $\Gamma$-predictor. The $\Gamma$-predictor $f$ has $\beta(\gamma)$ \emph{swap regret} on $D$ if for every $\gamma \in \im f$,
    \begin{align*}
        \mathbb{E}_{D}\left[ \ell(Y, \gamma)| f(X) = \gamma \right] - \min_{\hat{\gamma} \in \im \Gamma} \mathbb{E}_{D}\left[\ell(Y, \hat{\gamma}) | f(X) = \gamma \right] \le \beta(\gamma).
    \end{align*}
\end{definition}
One of the crucial insights in \citep{degroot_comparison_1983} was that mean-calibration relates to swap regret for the squared loss function.\footnote{\citep{globus2023multicalibration} and \citep{gopalan2024swap} extend this relationship to multi-calibration and swap agnostic learners.} In this section, we more generally show that if the property, which ought to be calibrated, is identifiable with a certain regular identification function, approximate calibration is equivalent to low swap regret. In particular, this extends the equivalence spelled out in Proposition~\ref{prop:Perfect Calibration via Loss Function or Identification Function} to the approximate case in Proposition~\ref{prop:Approximate Calibration and Low Swap Regret}, which requires certain regularity conditions on the identification function given in Definition~\ref{def:propertis of identification function}.
\begin{definition}[Properties of Identification Function]
\label{def:propertis of identification function}
    Let $\R \subseteq \reals$ be an interval and $\Gamma \colon \P \to \R$ an identifiable property with identification function $V\colon \Y \times \R \to \reals$.
    Let $\gamma, \gamma' \in \im \Gamma$. The identification function $V$ is called 
    \begin{description}
        \item[oriented]\label{eq:oriented identification function} if and only if for all $P \in P$, $ V(P, \gamma) > 0 \Leftrightarrow \gamma > \Gamma(P)$.
        \item[locally non-constant]\label{eq:locally nonconstant identification function} if and only if there exists $N > 0$ such that for all $P \in \P$, $N |\gamma - \Gamma(P)| \le |V(P, \gamma)|$.
        \item[locally Lipschitz]\label{eq:Lipschitz identification function} if and only if there exists $M > 0$ such that for all $P \in \P$, $|V(P, \gamma)| \le M |\gamma - \Gamma(P)|$.
    \end{description}
\end{definition}
Let us shortly discuss the regularity conditions. Orientedness is arguably a weak assumption. \citet{finocchiaro_convex_2018} have shown that in finite dimensions, \ie, $|\Y| < \infty$, there exists a reweighting of any identification function, such that the reweighted identification function is oriented and identifies the same property.\footnote{\citet{finocchiaro_convex_2018} show that the reweighted identification function is monotone increasing, which implies orientedness.} Furthermore, we obtain orientedness for free in the main theorem of \citep{steinwart2014elicitation}. The non-constantness and Lipschitzness assumptions have been considered, in its non-local variant, in the context of composite properties \citep[Assumption 5.2 and 5.3]{noarov_statistical_2023}. Note that $N \le M$ by definition, which intuitively captures the idea that an identification function cannot be more non-constant than Lipschitz-smooth. In Appendix~\ref{app: Examples of Properties with Regular Identification Functions} we provide several examples of properties beyond the mean which have identification functions which fulfill all of the above properties.
\begin{proposition}[Approximate $\Gamma$-Calibration Equivalent to Low Swap Regret]
\label{prop:Approximate Calibration and Low Swap Regret}
    Let $\Gamma \colon \P \rightarrow \R$ be an identifiable property with oriented identification function $V\colon \Y \times \R \to \reals$. Then, $\Gamma$ is elicitable with $\P$-consistent scoring function $\ell \colon \Y \times \R \rightarrow \reals$,
    \begin{align}
    \label{eq:loss function induced by identification function}
        \ell(y, \gamma) \coloneqq \int_{\gamma_0}^\gamma V(y,r) dr + \kappa(y),
    \end{align}
    for some $\gamma_0 \in \im \Gamma$ and $\kappa \colon \Y \rightarrow \reals$ having a finite expected value with respect to all $P \in \P$. Let $f$ be a $\Gamma$-predictor and $D$ a regular data distribution.
    \begin{enumerate}
        \item Let $V$ be locally Lipschitz on $\P$ with parameter $M$. If $f$ is $\alpha(\gamma)$-approximately $\Gamma$-calibrated, then $f$ has $\beta(\gamma)\coloneqq \frac{M}{2} \alpha(\gamma)^2$ swap regret.
        \item  Let $V$ be locally non-constant on $\P$ with parameter $N$. If $f$ has $\beta(\gamma)$ swap regret, then $f$ is $\alpha(\gamma)\coloneqq \sqrt{\frac{2}{N}\beta(\gamma)}$-approximate calibrated.
    \end{enumerate}
\end{proposition}
\begin{proof}
    Lemma~\ref{lemma:Identification Function Defines Consistent Loss Function} applied to $D_{Y | f(X) = \gamma}$.
\end{proof}
Hence, for the right choice of loss function, $\Gamma$-calibration and low swap regret can be used equivalently. Along the lines of \citep[Theorem 3.3]{gopalan2024swap}, we extend the equivalence relationship to group-wise definitions of calibration via \citep[Definition 2.10]{noarov_statistical_2023} in Appendix~\ref{appendix:Calibration-Swap-Regret Bridge under Groups}. However, we do \emph{not} achieve a direct generalization of \citep[Theorem 3.3 (1.) swap multicalibration $\iff$ (3.) swap-agnostic learner]{gopalan2024swap} by going from the mean (as in their work) to more general identifiable properties. Hence, we provide a bridge from calibration for general identifiable properties to swap regret notions, going beyond \citep{noarov_statistical_2023}. Nevertheless, the mean is a good example to illustrate Proposition~\ref{prop:Approximate Calibration and Low Swap Regret}.
\begin{example}
    Let $\Gamma$ be the mean and $V(y,\gamma)\coloneqq \gamma-y$ its identification function. In this case, $N =M = 1$. The mean is elicited by the squared loss,
    \begin{align*}
        \ell(y, \gamma) = \frac{1}{2}(y - \gamma)^2 = \int_{0}^\gamma V(y,r) dr + \frac{1}{2} y^2.
    \end{align*}
    Hence, a $\Gamma$-predictor $f$ on a regular data distribution $D$ has $\frac{1}{2}\alpha(\gamma)^2$ swap regret if, and only if it is $\alpha(\gamma)$-approximately calibrated (\cf \citep[Theorem 3.3]{gopalan2024swap}).
\end{example}

\subsection{\texorpdfstring{$\Gamma$}{Gamma}-Calibration is Inherited}\label{sec:self-realization-inherited}
There is a commonplace gap between \emph{optimization-level} $\Gamma$ and \emph{decision-level} properties $\Phi$ which trades off continuity for a granularity unnecessary for decision-making. Given this gap, we have to ask if self-realization with respect to $\Gamma$ extends to a related $\Phi$.
In the sequel, we show that self-realization, as $\Gamma$-calibration, is inherited by \emph{refined} properties (Definition~\ref{def:property refinement}).\footnote{The argument seems more involved then necessary. This is a consequence of carefully arguing about convex combinations of outcome distributions analogous to Proposition~\ref{prop:Distribution Calibration with Respect to Gamma implies Gamma-Calibration}.} That way, calibrated forecasts for optimization-level properties provide calibrated estimates for decision-level properties. In the language of decision making, $\Gamma$-calibration guarantees low swap regret for downstream decision makers.\footnote{This provides a major motivation for self-realization in the first place \citep{noarov2024calibration}.}
\begin{proposition}[$\Gamma$-Calibration is Inherited by Refined Properties]
\label{prop:Calibration is Inherited to Refined Properties}
    Let $\Gamma$ be a property, $D$ a regular data distribution and $f$ a $\Gamma$-predictor which is $\Gamma$-calibrated on $D$ with $|\im f| < \infty$, then for every property $\Phi$ with convex level sets and which is refined by $\Gamma$, the $\Phi$-predictor $\phi \circ f$ is $\Phi$-calibrated on $D$.
\end{proposition}
\begin{proof}
    By assumption, for all $\gamma \in \im f$,
    \begin{align*}
        \Gamma(D_{Y | f(X) = \gamma}) = \gamma.
    \end{align*}
    Hence, for all $\gamma \in \im f$,
    \begin{align*}
        \phi \circ \Gamma(D_{Y | f(X) = \gamma}) = \phi(\gamma).
    \end{align*}

    The conditional distribution of $Y$ given a fixed prediction $\gamma \in \im f$ is given by,
    \begin{align*}
        D_{Y | f(X)=\gamma} \coloneqq \frac{1}{w_\gamma } \int_\X \llbracket f(x) = \gamma \rrbracket \kappa_{Y|X}(x, \cdot) dD_X(x).
    \end{align*}
    where $w_\gamma = \int_\X \llbracket f(x) =\gamma \rrbracket dD_X(x)$ is the probability that $f(X) = \gamma$. Now, for every $v\in \im \phi \circ f$,
    \begin{align*}
        \frac{1}{w_v} \sum_{\gamma\in \phi^{-1}(v)} w_\gamma = 1
    \end{align*}
    where $w_v = \int_\X \llbracket \phi \circ f(x) =v \rrbracket dD_X(x)$ is a convex combination over the values $\gamma  \in \phi^{-1}(v)$ (which are finite because $\im f$ is finite).
    This gives,
    \begin{align*}
         &\frac{1}{w_v} \sum_{\gamma\in \phi^{-1}(v)}  D_{Y | f(X)=\gamma}w_\gamma \\
         &=\frac{1}{w_v} \sum_{\gamma\in \phi^{-1}(v)} \int_\X \llbracket x \colon f(x) =\gamma \rrbracket \kappa_{Y|X}(x, \cdot) dD_X(x)\\
         &=\frac{1}{w_v} \int_\X \sum_{\gamma\in \phi^{-1}(v)} \llbracket x \colon f(x) =\gamma \rrbracket  \kappa_{Y|X}(x, \cdot) dD_X(x)\\
         &=\frac{1}{w_v} \int_\X \llbracket \phi \circ f(x) =v \rrbracket \kappa_{Y|X}(x, \cdot) dD_X(x)\\
         &=D_{Y | \phi \circ f(X) = v}.
    \end{align*}
    Hence, $D_{Y | \phi \circ f(X) = v}$ is a convex combination of the distributions $D_{Y | f(X)=\gamma}$ for $\gamma \in \phi^{-1}(v)$. Since the level sets of $\phi \circ \Gamma$ are convex by assumption, for all $v\in \im \phi\circ f$,
    \begin{align*}
        \phi \circ \Gamma(D_{Y | \phi \circ f(X) = v}) = v.
    \end{align*}
\end{proof}
We conjecture that a more general statement for $|\im f|=\infty$ is true following an argument along the lines of \citep[Theorem 3.6]{noarov_statistical_2023}. For the approximate analogue of Proposition~\ref{prop:Calibration is Inherited to Refined Properties} we require the refined property to be Lipschitz.
\begin{proposition}[Approximate $\Gamma$-Calibration is Inherited by Refined Properties]
\label{prop:Approximate Calibration is Inherited to Refined Properties}
    Let $\Gamma$ be a property, $D$ a regular data distribution and $f$ a $\Gamma$-predictor which is $\Gamma$-calibrated on $D$, then for every property $\Phi$ which is refined by $\Gamma$ the $\Phi$-predictor $\phi \circ f$ is $\Phi$-calibrated on $D$. Assume further that $\phi$ is Lipschitz (with constant $K$), $|\im f| < \infty$ and $\Phi$ has convex level sets. 
    If $f$ is $\alpha(\gamma)$-approximately $\Gamma$-calibrated, then $\phi\circ f$ is $C\alpha'(\phi(v))$-approximately $\Phi$-calibrated, where $\alpha'(v) := \sup_{\gamma \in \phi^{-1}(v)}$.
\end{proposition}
\begin{proof}
    It is given that,
    \begin{align*}
        |\Gamma(D_{Y | f(X) = \gamma}) - \gamma| \le \alpha(\gamma),
    \end{align*}
    from which simply follows,
    \begin{align*}
        |\phi \circ \Gamma(D_{Y | f(X) = \gamma}) - \phi(\gamma)| \le K\alpha(\gamma).
    \end{align*}
    Then by a convexity argument following the proof of Proposition~\ref{prop:Calibration is Inherited to Refined Properties},
    \begin{align*}
        |\phi \circ \Gamma(D_{Y | \phi \circ f(X) = v}) - v| \le K\alpha'(v),
    \end{align*}
    where $\alpha'(v) = \sup_{\gamma \in \phi^{-1}(v)} \alpha(\gamma)$.
\end{proof}
The Lipschitz-condition on the property is strong. For instance, it rules out discrete properties. However, we will argue that in the case of elicitable properties we can still give guarantees about self-realization in terms of low swap regret (Proposition~\ref{prop:Low Swap Regret Guarantee for Refined Properties}). This statement requires an arguably mild regularity condition on the loss function.
\begin{definition}[Locally Outcome Lipschitz Loss Function]
\label{def: Locally Outcome Lipschitz Loss Function}
    Let $\Gamma \colon \P \rightarrow \R$. Suppose a loss function $\ell \colon \Y \times \R' \to \reals$ elicits a property $\Phi \colon \P \to \R'$. It is called \emph{locally outcome Lipschitz with respect to $\Gamma$} with constant $B$, if for all $v \in \im \Phi$, $\gamma \in \im \Gamma$, $P \in \Gamma^{-1}(\gamma)$ and $P' \in \P$,
    \begin{align}
        \label{eq:outcome Lipschitz loss metric}
        | \ell(P, v) - \ell(P', v) |  \le B m(\Gamma(P'), \gamma).
    \end{align}
\end{definition}
If $\Y$ is finite, $\Gamma$ is the full distribution, i.e., $\R = \Delta(\Y)$ and $\ell$ a bounded function, then $\ell$ is locally outcome Lipschitz (Appendix, Lemma~\ref{lemma:example for outcome lipschitzness}). Furthermore, the assumptions made in \citep{gopalan2024swap} imply Definition~\ref{def: Locally Outcome Lipschitz Loss Function} for $\Gamma$ being the mean on a binary outcome set.
\begin{proposition}[Low Swap Regret Guarantee for Refined Properties]
\label{prop:Low Swap Regret Guarantee for Refined Properties}
    Let $\Gamma \colon \P \to \R$ be a property, $D$ a regular data distribution and $f$ a $\Gamma$-predictor which is $\Gamma$-calibrated on $D$. Let $\Phi \colon \P \to \R'$ be a property which is refined by $\Gamma$, i.e., $\Phi \coloneqq \phi \circ \Gamma$ for some $\phi: \R \to \R'$.
    
    If $\Phi$ is elicitable with loss function $\ell$, locally outcome Lipschitz with respect to $\Gamma$ with constant $B$, and if $f$ is $\alpha(\gamma)$-approximately $\Gamma$-calibrated, then the swap regret of $\phi \circ f$ with respect to $\ell$ is bounded above by $2 B \mathbb{E}_{\gamma \sim D_{f(X)}}\left[ \alpha(\gamma)| \phi \circ \gamma = v \right]$ for all $v \in \im \phi \circ f$.
\end{proposition}
\begin{proof}
    Let $P \in \Gamma^{-1}(\gamma)$, by Definition~\ref{def: Locally Outcome Lipschitz Loss Function} Equation~\eqref{eq:outcome Lipschitz loss metric},
        for all $\gamma \in \im f$,
        \begin{align*}
            &\mathbb{E}_{D}\left[ \ell(Y, \phi(\gamma)) - \ell(Y, \Phi(D_{Y | \phi \circ f(X) = \phi(\gamma) }) | f(X) = \gamma \right]\\
            &= \ell(D_{Y | f(X)=\gamma }, \phi(\gamma)) - \ell(D_{Y | f(X)=\gamma }, \Phi(D_{Y | \phi \circ f(X) = \phi(\gamma) }))\\
            &= \ell(D_{Y | f(X)=\gamma }, v) - \ell(P, \Phi(D_{Y | \phi \circ f(X) = \phi(\gamma) }))\\
            &\quad + \ell(P, \Phi(D_{Y | \phi\circ f(X) = v })) - \ell(D_{Y | f(X)=\gamma }, \Phi(D_{Y | \phi\circ f(X) = v }))\\
            &\le \ell(D_{Y | f(X)=\gamma }, v) - \ell(P, \Phi(D_{Y | \phi\circ f(X) = v })) + B m(\Gamma(D_{Y | f(X)=\gamma }), \gamma)|\\
            &\le \ell(D_{Y | f(X)=\gamma }, v) - \ell(P, v) + \ell(P, v) - \ell(P, \Phi(D_{Y | \phi\circ f(X) = v }))\\
            &\quad + B m(\Gamma(D_{Y | f(X)=\gamma }), \gamma)\\
            &\le \ell(P, v) - \ell(P, \Phi(D_{Y | \phi\circ f(X) = v })) + 2 Bm(\Gamma(D_{Y | f(X)=\gamma }), \gamma)\\
            &\le 2 B\alpha(\gamma).
        \end{align*}
        It follows that, for all $v\in \im \phi\circ f$
        \begin{align*}
            &\mathbb{E}_{D}\left[ \ell(Y, \phi(\gamma)) - \ell(Y, \Phi(D_{Y | \phi \circ f(X) = \phi(\gamma) }) | \phi \circ f(X) = v \right]\\
            &\le 2 B \mathbb{E}_{\gamma \sim D_{f(X)}}\left[ \alpha(\gamma)| \phi \circ f(X) = v \right].
        \end{align*}
\end{proof}
An analogous proof holds for distributional calibration with respect to $\Gamma$ and a locally outcome Lipschitz loss function with respect to $\Gamma$. The proof is an adopted version of parts of the proof of the main theorem in \citep{gopalan2024swap}.

The inheritance of self-realization is closely connected to ``(swap) omniprediction'' introduced by \citep{gopalan2022omnipredictors,gopalan2024swap}. Informally, a (swap) omnipredictor provides predictions which allow for a simple post-processing to achieve low regret against a comparison base line for a large variety of losses. In particular, (swap) omnipredictors involve ``for all'' quantifiers about a set of losses.

\paragraph{Non-Refining (Swap) Omniprediction}
Proposition~\ref{prop:Approximate Calibration and Low Swap Regret} already involves an ``omni''-type statement. The statement holds for all loss functions defined following Equation~\eqref{eq:loss function induced by identification function}. Yet, there is no property refinement necessary.

\paragraph{Refining (Swap) Omniprediction}
The inheritance rule makes use of the refinement function $\phi$. This refinement function generalizes the post-processing function $k_\ell$ in \citep{gopalan2022omnipredictors}. In particular, Proposition~\ref{prop:Low Swap Regret Guarantee for Refined Properties} can be plugged together for a set of loss functions. For instance, let $f$ be a $\Gamma$-calibrated $\Gamma$-predictor on a data distribution $D$. Then, those predictions guarantee low swap regret for all loss functions which are locally outcome Lipschitz with respect to $\Gamma$. Hence, this is an ``omni''-type regret statement, related to the main theorem in \citep{gopalan2024swap}. In there, the authors focused on a very specific set of losses which are convex, Lipschitz and locally outcome Lipschitz with respect to the mean. However, their main theorem is generalized to multi-calibration with respect to groups. We push aside this further complexity in this work (cf. Section~\ref{What about groups}). Hence, our statement is less general in the scope of groupings, but more general in the scope of sets of loss functions and predicted properties.\footnote{Note that \citep{lu2025sample} as well consider generalizations of the the set of loss functions.}

\paragraph{}
The inheritance of self-realization is central to the rationale of using calibration. It provides low swap regret guarantees for a large class of loss functions. Each such loss function corresponds to an elicitable property, a ``decision-level'' property. For instance, this decision-level property can model a person who is about to decide whether to take or not to take their umbrella with them given a probabilistic forecast on rain. However, self-realization can have largely decoupled effects on the usefulness of the predictions for the ``decision-level'' properties, respectively the decision makers. A calibrated prediction can be more useful to one than to another decision maker (Appendix~\ref{appendix:Self-Realization Does Not Imply Equal to Usefulness for All}).

\section{Calibration as Precise Loss Estimation}\label{sec:calibration-as-precise-loss-estimation}
In contrast with the account laid out above, \citet{zhao_calibrating_2021} motivate calibration differently. They consider a predictor to be calibrated, if it can precisely\footnote{We deliberately refer to ``precise'' loss estimates as we focus on the internal consistency of the loss estimates and not the actual closeness to the truth (accuracy) \citep[p. 295]{everitt2002cambridge}.} estimate the average loss incurred to the individuals. Formally, the intuition can be captured by a type of loss outcome indistinguishability \citep{gopalan2022loss} following \citep{zhao_calibrating_2021}.
\begin{definition}[Decision Calibration (Modified from \citet{zhao_calibrating_2021})]
\label{def: decision calibration}
    Let $\L$ be a set of $\P$-consistent loss functions $\ell \colon \Y \times \R \to \reals$. We denote the property elicited by $\ell$ as $\Gamma_\ell$. Let $D$ be a regular data distribution and $f \colon \X \to \P$ a distribution predictor. The predictor $f$ is \emph{$\beta$-approximate decision calibrated for $\L$}, if for all $\ell \in \L$,
    \begin{align*}
        \left| \mathbb{E}_{(X,Y)\sim D}[\ell( Y, \Gamma_\ell(f(X))) - \mathbb{E}_{\hat{Y} \sim f(X)}[\ell(\hat{Y}, \Gamma_\ell(f(X)) ]] \right| \le \beta.
    \end{align*}
\end{definition}
In the original definition \citep[Definition 2]{zhao_calibrating_2021} the loss function and the elicited property can be independent of each other. We were not able to identify a convincing justification for this choice, hence we neglect this degree of freedom for our purposes. All involved decision-level properties $\Gamma_\ell$ are by definition elicitable.

In our definition we follow the presentation by \citet{fröhlich2024scoringrulescalibrationimprecise}. They back the intuition of ``calibration via precise loss estimation'' from an actuarial point of view referring to the ideal of \emph{actuarial fairness} \citep{arrow1978uncertainty}:
\begin{quote}
    The forecaster should offer actuarially fair insurance for the uncertain loss [...]. Actuarial fairness here means that in the long run, the forecaster neither loses nor profits from offering insurance under the data model. \citep[p. 12]{fröhlich2024scoringrulescalibrationimprecise} 
\end{quote}
Decision calibration fits to the narrative of ``trustworthiness'' via precise loss estimates pushed forward by \citet{kirchhof2024pretrained} and \citet{yoo2019learning} and which can be traced back to at least \citep{rukhin1988loss}. Note in those papers precise loss estimation is usually desired for every individual. In contrast, we focus on average precise loss estimation here. Decision calibration enforced on sub-groups based on input $X$ bridge between the two extremes: individual precise loss estimation versus average precise loss estimation (\cf Section~\ref{What about groups}).

Decision calibration, like all other so far named notions of calibration, grows on the ground of binary vanilla calibration (\cf Proposition~\ref{prop:Decision Calibration Equivalent to Mean Calibration for Binary Outcome Set}). We show that on binary outcome sets decision calibration with respect to the set of \emph{simple loss functions}, also known as ``cost-weighted misclassification losses'', is equivalent to binary vanilla calibration. Simple loss functions play a key role in defining the spectrum of proper loss function for binary outcome sets, \citep{schervish1989general} (for a modern proof see \citep[Theorem 16]{reid2011information} and an independent rediscovery \citep{kleinberg_u-calibration_2023}).
\begin{definition}[Simple Loss Function]
\label{def:simple loss function}
    A loss function $\ell \colon \{0,1\} \times \{ a,b\} \rightarrow \reals$ is called \emph{simple} if it has the form,
    \begin{align*}
        \ell_q(y, c) = q \llbracket y = 0, c = a\rrbracket + (1-q) \llbracket y = 1, c = b\rrbracket.
    \end{align*}
    The loss $\ell_q$ elicits the simple binary property,
    \begin{align*}
        \Gamma_q\colon \Delta(\Y) \to \{ a,b\}; P \mapsto \begin{cases}
        a \text{ if } P(Y=1) > q\\
        b \text{ if } P(Y=1) \le q.
    \end{cases}
    \end{align*}
\end{definition}
The equivalence of decision calibration and binary vanilla calibration has already been discussed in \citep[Theorem 1]{zhao_calibrating_2021}. However, their proof is hard to follow and makes heavy use of the independent choice of the decision function and the loss function, which we hesitated to commit to (see discussion below Definition~\ref{def: decision calibration}). Hence, our statement requires a weaker definition of decision calibration, but only holds on binary outcome sets and requires predictors with finitely many output values (arguably a mild assumption).\footnote{For an intuitive argument why Proposition~\ref{prop:Decision Calibration Equivalent to Mean Calibration for Binary Outcome Set} should hold see the example provided in \citep[Section 2.2]{holtgen2024practical}.}
\begin{proposition}[Decision Calibration Equivalent to Vanilla Calibration for Binary Outcome Set]
\label{prop:Decision Calibration Equivalent to Mean Calibration for Binary Outcome Set}
    Let $\Y = \{ 0,1\}$ and suppose $\X$ is arbitrary. Let $D$ be a regular data distribution on $\X$ and $\Y$. Suppose that $f \colon \X \rightarrow [0,1]$ is a mean-predictor, i.e., it predicts the probability of $Y =1$, and suppose that $|\im f| < \infty$.
    Then, $f$ is calibrated if, and only if, $f$ is decision calibrated with respect to all simple loss functions $\{ \ell_q\colon q \in [0,1]\}$.
\end{proposition}
\begin{proof}
    First, we consider two different cases. Let $f_{\inf} \coloneqq \inf_{\gamma \in \im f} \gamma$.
    \begin{enumerate}[(a)]
        \item If $f_{\inf} > 0$, then Lemma~\ref{lemma:matching averages by non-extremal predictions I} applies, hence $f$ matches the averages. 
        \item If $f_{\inf} = 0$, then Lemma~\ref{lemma:matching averages by non-extremal predictions II} applies because $\im f$ is finite, hence $f$ matches the averages.
    \end{enumerate}
    Lemma~\ref{lemma:When Decision Calibration is Equivalent to Calibration on Decision Sets} can be used to show that for all $q \in [0,1]$, $\mathbb{E}_{(X,Y) \sim D}[Y-f(X)| f(X) \le q] = 0$ and $\mathbb{E}_{(X,Y) \sim D}[Y-f(X)| f(X) > q] = 0$. It remains to show that for all $v \in \im f$, $\mathbb{E}_{(X,Y) \sim D}[Y-f(X)| f(X) = v] = 0$. If $v = 0$, this is already given. For every $v \in \im f \setminus \{ 0\}$ there exists $v' \in \im f \cup \{ 0\}$ such that $v' < v$ (because $\im f$ is finite). We have,
    \begin{align*}
        &\mathbb{E}_{(X,Y) \sim D}[Y-f(X)| f(X) \le v]\\
        &= P_D(f(X) \le v')\mathbb{E}_{(X,Y) \sim D}[Y-f(X)| f(X) \le v']\\
        &\quad + P_D(f(X) = v)\mathbb{E}_{(X,Y) \sim D}[Y-f(X)| f(X) = v]\\
        &= P_D(f(X) = v)\mathbb{E}_{(X,Y) \sim D}[Y-f(X)| f(X) = v] = 0,
    \end{align*}
    which gives the desired result.
\end{proof}
In light of Proposition~\ref{prop:Decision Calibration Equivalent to Mean Calibration for Binary Outcome Set}, the ``precise loss estimation''-account is another generalization of binary vanilla calibration whose semantic focus is distinct from the self-realization discussed in \S~\ref{Calibration as Self-realization of predictions}. Furthermore, precise loss estimation via decision calibration is also a child of the general notion of distribution calibration.

\subsection{Distribution Calibration Implies Decision Calibration}
Rather unsurprisingly, distribution calibration (Definition~\ref{def: Distribution Calibration with Respect to Gamma}) implies decision calibration (Definition~\ref{def: decision calibration}). This holds not only in the perfect case but as well in the case of approximate calibration if the loss function corresponding to the property of interest is bounded. Proposition~\ref{prop:Distribution Calibration Implies Decision Calibration} and Proposition~\ref{prop:Approximate Distribution Calibration Implies Approximate Decision Calibration} below have been proven in \citep[Proposition 2]{zhao_calibrating_2021}. However, we were not able to follow all steps of the proofs. For this reason we provide full arguments here.
\begin{proposition}[Distribution Calibration Implies Decision Calibration]
\label{prop:Distribution Calibration Implies Decision Calibration}
    Let $D$ be a regular data distribution on $\X \times \Y$, where $\Y$ is finite.
    Let $f\colon \X \to \P$ be a distributional predictor which is distribution calibrated with respect to $\Gamma$ on $D$. Then, $f$ is decision calibrated with respect to $\L_\Gamma \coloneqq \{ \ell \colon \ell \text{ is $\P$-consistent loss function for }\Gamma \}$.
\end{proposition}
\begin{proof}
     The predictor $f$ is \emph{distribution calibrated with respect to $\Gamma$} on $D$ if for every $\gamma \in \im \Gamma \circ f$,
    \begin{align}
        D_{Y|\Gamma \circ f(x) = \gamma\}}(Y=y) = \mathbb{E}_{X \sim D_X}[f_y(X)|\Gamma \circ f(X) = \gamma] , \quad \forall y \in \Y.\label{eq:dis-cal-on-phi}
    \end{align}
    where $f_y(x) \in [0,1]$ denotes the $y$-component of the prediction $f(x) \in \Delta(\Y)$ for $x \in \X$.

    Let $\ell \in \L_\Gamma$ be arbitrary. For every $\gamma \in \im \Gamma \circ f$, we have,
    \begin{align*}
        &\mathbb{E}_{(X,Y) \sim D}\left[\ell(Y, \gamma ) - \mathbb{E}_{\hat{Y} \sim f(X)}[\ell(\hat{Y}, \gamma ) ] | \Gamma\circ f(X) = \gamma \right]\\
        &= \mathbb{E}_{(X,Y) \sim D}\left[\ell(Y, \gamma)| \Gamma \circ f(X) = \gamma \right] - \mathbb{E}_{(X,Y) \sim D}\left[\mathbb{E}_{\hat{Y} \sim f(X)}[\ell(\hat{Y}, \gamma ) ]| \Gamma \circ f(X) = \gamma\right]\\
        &= \sum_{y \in \Y}D_{Y|\Gamma \circ f(X) = \gamma}(Y=y) \ell(y, \gamma) - \mathbb{E}_{X \sim D_X}\left[\sum_{y \in \Y}f_y(X)\ell(y, \gamma) \ \vline \ \Gamma \circ f(X) = \gamma\right]\\
        &= \sum_{y \in \Y}\left(D_{Y|\Gamma \circ f(X) = \gamma}(Y=y) - \mathbb{E}_{X \sim D_X}[f_y(X)|\Gamma \circ f(X) = \gamma]\right) \ell(y, \gamma) \\
        &= 0.
    \end{align*}
    by Equation~\eqref{eq:dis-cal-on-phi}. Hence, it follows,
    \begin{align*}
        \mathbb{E}_{(X,Y) \sim D}\left[\ell(Y, \gamma ) - \mathbb{E}_{\hat{Y} \sim f(X)}[\ell(\hat{Y}, \gamma ) ] | \Gamma\circ f(X) = \gamma \right] = 0.
    \end{align*}
    
\end{proof}
Further, this result for exact calibration generalizes to approximate calibration, subject to mild assumptions on the boundedness of the loss function $\ell$ which elicits the property of interest $\Gamma$.
\begin{proposition}[Approximate Distribution Calibration Implies Approximate Decision Calibration]
\label{prop:Approximate Distribution Calibration Implies Approximate Decision Calibration}
    Let $D$ be a regular data distribution on $\X \times \Y$, where $\Y$ is finite.
    Let $f\colon \X \to \P$ be a distributional predictor $\alpha(\gamma)$-approximately distribution calibrated with respect to $\Gamma$ on $D$. Then, $f$ is $C\mathbb{E}_{\gamma \sim D_{\Gamma \circ f(X)}}[\alpha(\gamma)]$-approximate decision calibrated with respect to $\L_{\Gamma,C} \coloneqq \{ \ell \colon \ell \text{ is $\P$-consistent loss function for }\Gamma \text{ and } \sup_{r \in \R}\sum_{y \in \Y}|\ell(y,r)| \le C < \infty\}$.
\end{proposition}
\begin{proof}
    The predictor $f$ is \emph{$\alpha(\gamma)$-approximate distribution calibrated with respect to $\Gamma$} on $D$ if for every $\gamma \in \im \Gamma \circ f$,
    \begin{align*}
        \left| D_{Y|\Gamma \circ f(X) = \gamma}(Y=y) - \mathbb{E}_{D}[f_y(X)|\Gamma \circ f(X) = \gamma] \right|\le \alpha(\gamma), \quad \forall y \in \Y.
    \end{align*}
    where $f_y(x) \in [0,1]$ denotes the $y$-component of the prediction $f(x) \in \Delta(\Y)$ for $x \in \X$.

    Let $\ell \in \L_{\Gamma, C}$ be arbitrary but fixed. For all $\gamma \in \im \Gamma \circ f$, we have,
    \begin{align*}
        &\left|\mathbb{E}_{(X,Y) \sim D}\left[\ell(Y, \gamma ) - \mathbb{E}_{\hat{Y} \sim f(X)}[\ell(\hat{Y}, \gamma ) ] | \Gamma\circ f(X) = \gamma \right] \right|\\
        &= \left|\sum_{y \in \Y}\left(D_{Y|\Gamma \circ f(X) = \gamma}(Y=y) - \mathbb{E}_{X \sim D_X}[f_y(X)|\Gamma \circ f(X) = \gamma]\right) \ell(y, \gamma)\right| \\
        &\le \sum_{y \in \Y}\left|D_{Y|\Gamma \circ f(X) = \gamma}(Y=y) - \mathbb{E}_{X \sim D_X}[f_y(X)|\Gamma \circ f(X) = \gamma]\right| |\ell(y, \gamma)|\\
        &\le \alpha(\gamma) \sum_{y \in \Y}  |\ell(y, \gamma)| \le C \alpha(\gamma).
    \end{align*}
    The result follows by taking the expectation over $\gamma \sim D_{\Gamma \circ f(X)}$.

    For detailed steps from the first to the second term see the proof of Proposition~\ref{prop:Distribution Calibration Implies Decision Calibration}.
\end{proof}

\subsection{Precise Bayes Risk Estimation}
Analogous to self-realization we obtain a clearer picture of the purpose of decision calibration when rewriting in terms of properties. To this end, we introduce Bayes pairs. A Bayes pair unifies the minimizer (elicited property) and the minimization value (Bayes risk) of a loss function.\footnote{The importance of eliciting Bayes pairs is also studied by \citep{fissler_higher_2016,frongillo_elicitation_2021,finocchiaro_unifying_2021}.} Hence, following the account of precise loss estimation the Bayes risk is of central importance here. Being calibrated in the sense of precise loss estimation, is then tantamount to precise estimation of the Bayes risk.
\begin{definition}[Bayes Pair \citep{embrechts2021bayes}]
    Let $\ell \colon \Y \times \R \to \reals$ be a loss function. The property pair $(\Phi_\ell, \Theta_\ell)$ is called a Bayes pair, if, for all $P \in \P$,\footnote{Under the restriction that $\R$ is compact both properties are guaranteed to be measurable \citep[Theorem 18.19]{aliprantis2006infinite}.}
    \begin{align*}
        \Phi_\ell(P) = \argmin_{\phi  \in \R} \mathbb{E}_{Y \sim P}[\ell(Y, \phi)],\\
        \Theta_\ell(P) = \min_{\phi \in \R} \mathbb{E}_{Y \sim P}[\ell(Y, \phi)].
    \end{align*}
\end{definition}
By definition, $\Phi_\ell$ is elicitable. That the Bayes risk $\Theta_\ell$ is elicitable on level sets of $\Phi_\ell$ is less obvious. To see this, observe that there exists a $\Phi_\ell^{-1}(\phi)$-consistent loss function for $\Theta_\ell$ for every $\phi \in \im \Phi_\ell$ \citep{embrechts2021bayes}. 
We proceed by naively defining precise Bayes risk estimators.
\begin{definition}[Precise Bayes Risk Estimator]
\label{def:Precise Bayes Risk Estimator}
    Let $\ell \colon \Y \times \R \to \reals$ be a loss function with corresponding Bayes pair $(\Phi_\ell, \Theta_\ell)$. Let $D$ be a regular data distribution. A $(\Phi_\ell, \Theta_\ell)$-predictor $f = (g,h)$ is a \emph{$\beta$-precise Bayes risk estimator} if,
    \begin{align*}
         \left| \mathbb{E}_{D}[\ell( Y, g(X)) - h(X)] \right| \le \beta.
    \end{align*}
\end{definition}
This definition is extremely weak. The predictor $h$ estimates the loss but only has to fit it on average. The definition can be significantly strengthened if precise Bayes risk estimation is demanded on subgroups or individuals (\cf Section~\ref{What about groups}).
Note that such Bayes risk predictors exist in current literature. For instance, they are called ``uncertainty module'' \citep{kirchhof2024pretrained} or ``loss prediction module'' \citep{yoo2019learning}. It follows rather immediately that a decision calibrated full probability predictor is a precise Bayes risk estimator.
\begin{proposition}[Decision Calibration Implies Precise Bayes Risk Estimation]
    \label{prop:Decision Calibration Implies Precise Bayes Risk Estimation}
    Let $\L$ be a set of $\P$-consistent loss functions $\ell \colon \Y \times \R \to \reals$ with corresponding Bayes pairs $(\Phi_\ell, \Theta_\ell)$. Let $D$ be a regular data distribution and $f \colon \X \to \P$ a $\beta$-approximate decision calibrated predictor for $\L$. Then, for all $\ell \in \L$ the predictor $(\Phi_\ell \circ f, \Theta_\ell \circ f)$ is a $\beta$-precise Bayes risk estimator.
\end{proposition}
\begin{proof}
    We can rewrite the criterion of decision calibration as,
\begin{align}
\label{eq:Decision Calibration Implies Precise Bayes Risk Estimation}
    \left|\mathbb{E}_{D}[\ell( Y, \Phi_\ell(f(X))) - \mathbb{E}_{\hat{Y} \sim f(X)}[\ell(\hat{Y}, \Phi_\ell(f(X)) ]]\right| =\left| \mathbb{E}_{D}[\ell( Y, \Phi_\ell(f(X))) - \Theta_\ell(f(X))] \right| \le \beta.
\end{align}
\end{proof}

\subsection{When Self-Realization Implies Precise Loss Estimation}\label{sec:When Self-Realization Implies Precise Loss Estimation, and Vice-Versa}
In the sections before we explored the landscape of calibration as self-realization and as precise loss estimation. We introduced two prototypical notions of calibration using properties for each of the two worlds: property calibration and precise Bayes risk estimation. In the following section, we present a short series of results which shed light on the complicated relationship between the notions.

If we have access to a predictor which predicts a Bayes pair and is property-calibrated with respect to the involved Bayes risk property, then the predictor is a precise Bayes risk estimator (Proposition~\ref{prop:Self-Realization Implies Precise Loss Estimation}). That is, property calibration for the property $(\Gamma, \Theta)$ implies precise Bayes risk estimation. 
\begin{proposition}[Self-Realization Implies Precise Loss Estimation]
\label{prop:Self-Realization Implies Precise Loss Estimation}
    Let $(\Phi_\ell, \Theta_\ell)$ be a Bayes pair corresponding to a loss function $\ell$ and $D$ a regular data distribution. If a $(\Phi_\ell, \Theta_\ell)$-predictor $f = (g,h)$ is $\alpha(\phi, \theta)$-approximate $(\Phi_\ell, \Theta_\ell)$-calibrated, then $f$ is a $\alpha$-precise Bayes risk estimator, where $\alpha \coloneqq \mathbb{E}_{(\phi, \theta) \sim Q}[|\alpha(\phi, \theta)|]$ and $Q \coloneqq D_{(g(X), h(X))}$
\end{proposition}
\begin{proof}
    \begin{align*}
        &|\mathbb{E}_{D}[\ell( Y, g(X)) - h(X)] |\\
        &=|\mathbb{E}_{(\phi, \theta) \sim Q}[\mathbb{E}_{(X,Y) \sim D|g(X) = \phi, h(X) = \theta}[\ell( Y, \phi)] - \theta ] |\\
        &\le \mathbb{E}_{(\phi, \theta) \sim Q}[| \Theta_\ell(D_{Y |g(X) = \phi, h(X) = \theta}) - \theta |]\\
        &\le \mathbb{E}_{(\phi, \theta) \sim Q}[|\alpha(\phi, \theta)|] = \alpha.
    \end{align*}
\end{proof}
The implication cannot be generally extended into a bi-implication as the following example shows. In that sense, self-realization is stronger than precise loss estimation.
\begin{example}[On the Necessity of Self-Realization for Precise Loss Estimation]
\label{ex:On the Necessity of Self-Realization for Precise Loss Estimation}
    Let $\Y = \reals$. Let $\ell(y, \gamma) = (y - \gamma)^2$ be the squared loss. The Bayes pair corresponding to $\ell$ is $(\mathbb{M}, \mathbb{V})$ the mean and variance. 
    Let $\X = \{ x_1, x_2\}$ and $D$ a data distribution with a fixed, but arbitrary marginal $D_X$ with full support and conditional distributions $D_{Y | X = x_1}$ with mean $1$ and uncentered second moment $\mathbb{E}_{Y \sim D_{Y | X = x_1 }} [Y^2] = v$ respectively $D_{Y | X = x_2}$ with mean $0$ and uncentered second moment $\mathbb{E}_{Y \sim D_{Y | X = x_2}} [Y^2] = v$. Let $f(x_1) = 0$ and $f(x_2) = 1$, $g(x_1) = v$ and $g(x_2) = v + 1$. 

    In this example, $(f,g)$ is a precise Bayes risk estimator, but $(f,g)$ is not $(\mathbb{M},\mathbb{V})$-calibrated, nor $f$ is $\mathbb{M}$-calibrated, nor $g$ is $\mathbb{V}$-calibrated.
    
    To see that $(f,g)$ is a precise loss estimator:
    \begin{align*}
        &\mathbb{E}_{D}[(Y - f(X))^2 - g(X)]\\
        &= D_X(\{ x_1\} ) \mathbb{E}_{Y \sim D_Y}[Y^2 - v| X = x_1] + D_X(\{ x_2\}) \mathbb{E}_{Y \sim D_Y}[(Y-1)^2 - v - 1| X = x_2]\\
        &= D_X(\{x_1\}) \E_{Y \sim D_Y} [Y^2| X = x_1] -v + D_X(\{ x_2\}) \E_{Y \sim D_Y} [Y^2 - 2Y +1| X = x_2] - v-1 \\
        &= 0
    \end{align*}
    On the other hand, $f$ is clearly not $\mathbb{M}$-calibrated. Furthermore, $g$ is not $\mathbb{V}$-calibrated, because,
    \begin{align*}
        \mathbb{V}(D_{Y|g(X) = v}) &= \mathbb{E}_{Y \sim D_Y}[(Y - 1)^2|X =x_1]\\
        &= \mathbb{E}_{Y \sim D_Y}[Y^2|X =x_1] - 2\mathbb{E}_{Y\sim D_Y}[Y|X =x_1] + 1\\
        &= v -1 \neq v,
    \end{align*}
    and
    \begin{align*}
        \mathbb{V}(D_{Y|g(X) = v + 1}) &= \mathbb{E}_{D_Y}[Y^2|X =x_2]\\
        &= v \neq v + 1.
    \end{align*}
\end{example}
However, precise loss estimation always requires one to predict not only the property of interest but as well the corresponding Bayes risk. For instance, a pure mean predictor $f$ as in Example~\ref{ex:On the Necessity of Self-Realization for Precise Loss Estimation} can be perfectly calibrated but still misses the information for being a precise Bayes risk estimator. In that sense, $\Gamma$-calibration and precise Bayes risk estimation are incommensurable as they are properties of different predictors. This explains why distribution calibration, which requires full prediction estimates, subsumes both: all properties of a distribution including the Bayes risks are refined by the full distribution.

Nevertheless, Example~\ref{ex:On the Necessity of Self-Realization for Precise Loss Estimation} can be strengthened. If $\Y$ is binary then the mean predictor identifies the full distribution. But even then, there exist predictors which are perfectly decision calibrated with respect to the squared loss (which implies precise Bayes risk estimation \cf~Proposition~\ref{prop:Decision Calibration Implies Precise Bayes Risk Estimation}), but still are not mean calibrated.
\begin{example}[Squared Loss Decision Calibration Does not Imply Mean Calibration]
\label{ex:Squared Loss Decision Calibration Does not Imply Mean Calibration}
    Let $\Y = \{0,1 \}$ and $D$ be a regular data distribution on $\X = \reals$ and $\Y$.  Let $\ell(y, \gamma) = (y - \gamma)^2$ be the squared loss and $\mathbb{M}$ the mean. Let $f \colon \X \to [0,1]$ be a mean-predictor. Hence, $f$ predicts the probability $p(x) \coloneqq D_{Y|X=x}(\{1\}) = p(x)$. Let us rewrite the condition for decision calibration with respect to $\ell$,
    \begin{align*}
        \mathbb{E}_{Y\sim D_{Y | \{ X = x \}}}[(Y - f(x))^2|X = x] &= \mathbb{E}_{\hat{Y} \sim f(x)}[(\hat{Y} - f(x))^2|X = x]\\
        \Leftrightarrow p(x) (1- f(x))^2 + (1-p(x)) f(x)^2 &= f(x) (1- f(x))^2 + (1-f(x)) f(x)^2\\
        \Leftrightarrow p(x) - 2p(x)f(x) + f(x)^2 &= f(x) - 2f(x)^2 + f(x)^2\\
        \Leftrightarrow p(x) - 2p(x)f(x) &= f(x) - 2f(x)^2\\
        \Leftrightarrow p(x) ( 1 - 2f(x)) &= f(x) (1- 2f(x)),
    \end{align*}
    which if $f(x)\neq \frac{1}{2}$ is equivalent to
    $p(x) = f(x)$. Hence, decision calibration with respect to $\ell$ is given by,
    \begin{align*}
        \mathbb{E}_{X}[(p(x) - f(x)) ( 1 - 2f(x))] = 0,
    \end{align*}
    which, for instance, can be fulfilled by a predictor $f(x) = \frac{1}{2}$. Since $\mathbb{\mathbb{M}}$-calibration and no swap regret with respect to $\ell$ is equivalent (Proposition~\ref{prop:Perfect Calibration via Loss Function or Identification Function}), squared loss decision calibration is not sufficient for mean-calibration in the binary setting.
\end{example}

\section{What about groups?}\label{What about groups}
Certain frameworks like fair machine learning demand for calibration on subgroups defined through the input space $\X$. Hence, one can understand calibration in fair machine learning, as asking for ``equally good self-realization of predictions on sensitive subgroups'' or ``equally precise loss estimation on sensitive subgroups''.
The notions of calibration we considered can easily be extended to subgroups. For instance, by taking a supremum over subgroups in Definition~\ref{def: Gamma-calibration} one recovers the notion of multi-calibration used in \citep{noarov2024calibration}. Respectively, a precise Bayes risk estimator following Definition~\ref{def:Precise Bayes Risk Estimator} can be strengthened by checking for precise Bayes risk estimation on all relevant subgroups. Note that the choice of groups on which a certain notion of calibration ought to hold is \emph{independent} to the choice of the notion itself. 

In the extreme case, those subgroups could be enforced to the level of individuality \citep{zhao2020individual, luo2022local}. Then ``trustworthiness'' meets ``usefulness'' again. The Bayes optimal predictor is the only predictor which achieves ``trustworthiness'' on all individuals. In particular, it is possible to equalize ``trustworthiness'' on arbitrary subgroups without losing on ``usefulness''. However, it is not possible to equalize ``usefulness'' on arbitrary subgroups without losing on ``trustworthiness'' \citep[Proposition 4 \& 5]{barocas-hardt-narayanan}.

Concluding, all considerations made on the semantics of the three different types of calibration considered in this work plus their formal relationships directly transfer to the fairness setting with multiple groups.

\section{Conclusion}
This work has started with the promise to provide semantical structure in the chaotic world of calibration notions. We consider as our main contribution to semantically and formally separate two concerns calibration wants to solve: self-realization and precise loss estimation. Both accounts differently motivate a certain type of calibration. However, they can be bridged by distribution calibration.

How do existing notions relate to the three different types? Table~\ref{tab:notions of calibration - strict derivatives} and Table~\ref{tab:notions of calibration - philosophy} summarizes a list of existing notions and their categorization in the three different types. Note that we make a distinction between formal strict derivatives, i.e., our suggested notion generalizes the notion in the list, and definitions which follow the type of definition more broadly.

What is the learning for a practitioner? We have \emph{not} provided any new algorithm nor have we suggested \emph{the} notion of calibration.
But, our works shapes the debate on what is the right notion. It highlights implications (see Figure~\ref{fig: calibration backbone} and Figure~\ref{fig: calibration backbone approximate}), incommensurabilities (see Section~\ref{sec:When Self-Realization Implies Precise Loss Estimation, and Vice-Versa}) and equivalences (Proposition~\ref{prop:Recovering Distribution Calibration}, Proposition~\ref{prop:approximate distribution calibration implies gamma calibration} and Proposition~\ref{prop:Decision Calibration Equivalent to Mean Calibration for Binary Outcome Set}).
If those relationships are ignored, this may lead to unforeseen, unintentional consequences or may unnecessarily complicate a prediction problem\footnote{For instance, in the example given in Section~\ref{an exemplary motivation} any calibration beyond the ranking is unnecessary, since the policy is bound to the rankings only \citep{perdomo_difficult_2023}.}.
Our work helps to define a set of questions, which are relevant to formalizing a real-world prediction problem. What are the predictions we can or should get? What and who are the agents which use the predictions? Is one of the two accounts laid out desired as quality criterion of the predictions?
Finally, our work advocates to leave the formal distinctions between evaluation metrics behind (Proposition~\ref{prop:Approximate Calibration and Low Swap Regret}) and re-focus on the actual purpose of evaluation: estimation of usefulness, self-realization or guarantees to the estimation of usefulness. Particularly, calibration is only a part of this bigger picture.

\begin{table}
    \centering
    \begin{tabular}{l|lc}
\textbf{Formally Strict Derivative}      & Name                                                                                                    & Reference                        \\ \hline
\multirow{3}{*}{Distribution Calibration} & class-wise calibration                                                                                  & \citep{kull2019beyond}         \\
                                          & confidence calibration                                                                                  & \citep{guo_calibration_2017}     \\
                                          & \makecell[l]{event-conditional unbiasedness \\ (without groupings via $\X$)} & \citep{noarov2023high}           \\ \hline
\multirow{3}{*}{$\Gamma$-Calibration}     & $\Gamma$-calibration (without groupings via $\X$)                                                       & \citep{noarov_statistical_2023}  \\ 
                                          & $T$-calibration                                                                                         & \citep{gneiting_regression_2023} \\
                                          & quantile calibration                                                                                    & \citep{kuleshov2022calibrated}   \\ \hline
Decision Calibration     & decision calibration                                                                                    & \citep{zhao_calibrating_2021}    \\
\end{tabular}
    \caption{Categorization of existing notions of calibration in the three different types following by which type the notion is formally generalized. This table does not claim to contain a complete list of all definitions of calibration.}
    \label{tab:notions of calibration - strict derivatives}
\end{table}
\begin{table}
    \centering
\begin{tabular}{l|lc}
\textbf{Comparable Type of Definition} & Name                                        & Reference                              \\ \hline
Distribution Calibration              & distribution calibration                    & \citep{song2019distribution}         \\ \hline
\multirow{3}{*}{$\Gamma$-Calibration} & marginal coverage calibration               & \citep{raeth2025marginal}            \\ 
                                      & $u$-calibration                             & \citep{kleinberg_u-calibration_2023} \\
                                      & smooth calibration                          & \citep{foster2018smooth}            \\
                                      & distance to calibration                     & \citep{blasiok_unifying_2023}        \\\hline
\multirow{2}{*}{Decision Calibration} & loss outcome indistinguishability           & \citep{gopalan2022loss}        \\
                                      & \makecell[l]{IP-calibration \\ (without groupings via $\X$)} & \citep{fröhlich2024scoringrulescalibrationimprecise}                     \\
\end{tabular}
    \caption{Categorization of existing notions of calibration in the three different types following a broader comparability of the account. This table does not claim to contain a complete list of all definitions of calibration.}
    \label{tab:notions of calibration - philosophy}
\end{table}

\section{Acknowledgements}
The authors are very grateful to Benedikt Höltgen and Rajeev Verma for feedback on an earlier draft of this work. Thanks to the International Max Planck Research School for Intelligent System (IMPRS-IS) for supporting Rabanus Derr. Thanks to James Bailie for giving a reason to visit Cambridge, MA where Jessie and Rabanus met the first time. Part of the research happened when Rabanus was visiting Aaron Roth supported by a DAAD IFI-Stipend.
Rabanus Derr and Robert C. Williamson were funded in part by the Deutsche Forschungsgemeinschaft (DFG, German Research Foundation) under Germany’s Excellence Strategy — EXC number 2064/1 — Project number 390727645; they were also supported by the German Federal Ministry of Education and Research (BMBF): Tübingen AI Center.
Jessie Finocchiaro was supported in part by the National Science Foundation under Award No. 2202898.

\appendix

\section{Distribution Calibration with Respect to all Binary Properties Implies Full Distribution Calibration}
\label{app: dist calibration wrt to all binary properties}

It follows relatively immediately that distribution calibration with respect to the full distribution implies distribution calibration with respect to any property. The reverse direction requires one to think about a \emph{set} of properties. We show that if a forecast is perfectly distribution calibrated with respect to all binary, elicitable properties, it is distribution calibrated with respect to the full distribution. This statement requires some technical-looking assumption on separability via hyperplanes. If the number of different predicted forecasts is finite, which is arguably a mild assumption in practice, then this separability assumption holds.
\begin{lemma}
    If $|\im f| < \infty$, then for every $p \in \im f$ there exists a hyperplane $H = \{x \in \R^{|\Y|} \colon \langle a, x\rangle = b \}$ such that $p \in H$ and $0 < \epsilon \le \inf_{h \in H}\| h - q\|$ for every $q \in \im f \setminus \{ p\}$.
\end{lemma}
\begin{proof}
    The proof follows by contradiction. Let us assume that for some point $p \in \im f$ there is no hyperplane such that $p \in H$ and $0 < \epsilon \le \inf_{h \in H}\| h - q\|$ for every $q \in \im f \setminus \{ p\}$. That is, for every hyperplane $H$ such that $p \in H$ there is $q_H \in \im f \setminus \{ p\}$ with $\inf_{h \in H}\| h - q\| < \epsilon$ for every $\epsilon > 0$. It follows that such $q_H \in H$, because $\inf_{h \in H}\| h - q\| = 0$. Now, since there are uncountably many such hyperplanes $H$ but $|\im f| < \infty$ we get into a contradiction. 
\end{proof}
\begin{proposition}[Recovering Distribution Calibration by Distribution Calibration with respect to all Binary, Elicitable Properties]
\label{prop:Recovering Distribution Calibration}
    Let $D$ be a regular data distribution on $\X \times \Y$, where $\Y$ is finite. Let $f \colon \X \to \P$ be a distributional predictor.
    Suppose that $p_D(f(X) = q) > 0$ for all $q \in \im f$.
    
    If the predictor $f$ is distribution calibrated with respect to all binary elicitable properties $\Phi \colon \Delta(\Y) \to \{ 0,1\}$, then it is distribution calibrated with respect to the full distribution on $D$.
\end{proposition}
\begin{proof}
    Let us pick an arbitrary $p \in \im f$. We show that there exist two elicitable, binary properties $\underline{\Phi}_p, \overline{\Phi}_p$ such that,
    \begin{align*}
        \underline{\Phi}_p( q ) &= \overline{\Phi}_p( q ), \quad \forall q \in \im f \setminus \{ p\}\\
        \underline{\Phi}_p( p ) &\neq \overline{\Phi}_p( p).
    \end{align*}
    Then, we can argue that there is $a \in \{ 0,1\}$ such that
    \begin{align*}
        \mathbb{E}_{D}[\llbracket Y = y \rrbracket - f_y(X)| \underline{\Phi}_p \circ f(X) = a] = 0\\
        \mathbb{E}_{D}[\llbracket Y = y \rrbracket - f_y(X)| \overline{\Phi}_p \circ f(X) = a] = 0,
    \end{align*}
    where, with some $c \in [0,1]$,
    \begin{align*}
        &\mathbb{E}_{D}[\llbracket Y = y \rrbracket - f_y(X)| \overline{\Phi}_p \circ f(X) = a] \\
        &= c \mathbb{E}_{D}[\llbracket Y = y \rrbracket - f_y(X)| \underline{\Phi}_p \circ f(X) = a] + p_D(f(X) = p) \mathbb{E}_{D}[\llbracket Y = y \rrbracket - f_y(X)| f(X) = p].
    \end{align*}
    Hence, because $p_D(f(X) = p) > 0$,
    \begin{align*}
        \mathbb{E}_{D}[\llbracket Y = y \rrbracket - f_y(X)| f(X) = p] = 0.
    \end{align*}
    Since $p \in \im f$ was picked arbitrarily, this shows the claim.

    Nevertheless, it remains to argue that there exist such two elicitable, binary properties $\underline{\Phi}_p, \overline{\Phi}_p$ as demanded above. We show this by the proofing the existence of two hyperplanes such that all points $q \in \im f$ except for $p$ stay on the same side of both hyperplanes, but $p$ switches the side.
    As proven in \citep[Theorem 1]{lambert_elicitation_2019}, properties are elicitable if and only if the level sets form a power diagram, i.e., the intersection of the simplex $\Delta(\Y) \subseteq \mathbb{R}^{|\Y|}$ with a Voronoi diagram. For our purpose, it suffices to know that a hyperplane in $\mathbb{R}^{|\Y|}$ defines a binary, elicitable property. In fact, every binary elicitable property can be expressed through a corresponding hyperplane.

    Let
    \begin{align*}
        \underline{\Phi}_p(q) =\begin{cases}
            0 \text{ if } \langle a, q \rangle > b - \frac{\epsilon \| a\|}{2}\\
            1 \text{ otherwise. }
        \end{cases}.
    \end{align*}
    and
    \begin{align*}
        \overline{\Phi}_p(q) =\begin{cases}
            0 \text{ if } \langle a, q\rangle > b + \frac{\epsilon \| a\|}{2}\\
            1 \text{ otherwise.}
        \end{cases}.
    \end{align*}

    Then,
    \begin{align*}
        \underline{\Phi}_p( p ) = 0 \neq 1 = \overline{\Phi}_p( p ),
    \end{align*}
    but for all $q \in \im f \setminus\{ p\}$,
    \begin{align*}
        \underline{\Phi}_p( p ) = \overline{\Phi}_p( p ),
    \end{align*}
    because if $\underline{\Phi}_p(q) = 0$ (the analogous argument holds for $\underline{\Phi}_p(q) = 1$), then
    \begin{align*}
        \langle a, q \rangle > b - \frac{\epsilon \| a\|}{2},
    \end{align*}
    which implies,
    \begin{align*}
        \langle a, q \rangle \ge b + \|a \| \epsilon > b + \frac{\epsilon \| a\|}{2}
    \end{align*}
    by formula of distance to a hyperplane,
    \begin{align*}
        \epsilon &\le \inf_{h \in H}\| h - q\| = \frac{|\langle a, q \rangle - b|}{\| a\|}
        &\Leftrightarrow \| a\| \epsilon + b &\le \langle a, q \rangle.
    \end{align*}
\end{proof}

\section{Calibration and the Reflection Principle}
\label{calibration and the reflection principle}
Central to vanilla calibration is the self-referential aspect in the conditioning, i.e., conditioning on the prediction $p$ should make the outcome distribution equal to $p$. The self-referential aspect is similar in nature to a widely discussed principle for deference of belief in philosophy, the \emph{reflection principle} \citep{molinari2023trust}.\footnote{\citet{van1984belief} introduced the reflection principle to argue that reflection of subjective beliefs of future selves is a requirement for rationality of the agent.}

In this appendix, we shortly show that a modified variant of the reflection principle on binary outcome sets is equivalent to binary vanilla calibration.
\begin{definition}[Reflection Principle \citep{van1984belief}]
    Let $\X \subseteq \mathbb{R}^d$ and $\Y = \{ 0,1\}$. Let $D$ be a regular data distribution on $\X \times \Y$. Let $Q \colon \X \times \Y \rightarrow \Delta(\X \times \Y)$ be measurable. Then \emph{$Q$ reflects $D$} if and only if,
    \begin{align*}
        D(\X \times \{ 1\}| \{ (x,y) \in \X \times \Y \colon Q_{(x,y)}(\X \times \{ 1\}) = r\}) = r,
    \end{align*}
    for all $r \in [0,1]$.\footnote{We assume that when conditioning on a measure zero event the condition is always fulfilled.}
\end{definition}
\begin{proposition}[Calibration is Equivalent to Reflection Principle on Binary $\Y$]
    Let $\X \subseteq \mathbb{R}^d$ and $\Y = \{ 0,1\}$. Let $D$ be a regular data distribution on $\X \times \Y$. Let $f \colon \X \rightarrow [0,1]$ be a mean-predictor. Let $Q \colon \X \times \Y \rightarrow \Delta(\X \times \Y)$  such that for all $(x,y) \in \X \times \Y$, $Q_{(x,y)} \colon \X \times \{ 1\} \mapsto f(x)$\footnote{Such a function $Q$ exists, because there exist such $Q_{(x,y)}$ for all $(x,y) \in \X \times \Y$ by Hahn-Banach's extension theorem.}. The predictor $f$ is calibrated on $D$, if and only if $Q$ reflects $D$.
\end{proposition}
\begin{proof}
    Note that for all $r \in \im f$
    \begin{align*}
        &D(\X \times \{ 1\}| \{ (x,y) \in \X \times \Y \colon Q_{(x,y)}(\X \times \{ 1\}) = r\})\\
        &= D(\X \times \{ 1\}| \{ (x,y) \in \X \times \Y \colon f(x)= r\})\\
        &= D_{Y | \{ x \in \X \colon f(x) = r\}} ( Y  = 1 )\\
        &= \mathbb{E}_{(X,Y) \sim D}[Y | f(X) = r].
    \end{align*}
    Hence, the equivalence follows.
\end{proof}
In a footnote, \citet[p. 276]{seidenfeld1985calibration} suggest a definition of ``subjective calibration'' which resembles the reflection principle as defined here. Unfortunately, we were not able to identify the work the author refers to.

We don't delve deeper into interpretations of the reflection principle. Vaguely following the idea of the reflection principle as trust statement \citep{molinari2023trust} one might put calibration in a debatable slogan: \emph{Calibration guarantees that nature ``trusts'' the predictions of the predictor.} However, this does not make agents which get informed by the calibrated prediction to ``trust'' the predictions.\footnote{It is unclear whether the agents ``trust'' nature. Additionally, trust seems to not be generally transitive \citep{christianson1996isn, falcone2012trust}.} More humbly formulated, a central element of calibration is the reflection or self-realization of a certain aggregated quantity.

\section{A Note on Sensibility to Calibration}
\label{app: A note on sensibility to calibration}
The reader familiar with \citep{noarov_statistical_2023} might question the use of a general property $\Gamma$ in Definition~\ref{def: Gamma-calibration}. The authors introduce ``sensibility'' for calibration which requires that the true property predictor is calibrated with respect to Definition~\ref{def: Gamma-calibration}. As the authors show therein, a property is only sensible for calibration if and only if the level sets (pre-images) of property values are convex.
\citet{steinwart2014elicitation} in turn, have proven that if the property is strictly locally non-constant, a minor technical assumption, and continuous, then convexity of the level sets of the property is equivalent to elicitability (respectively identifiability). If the property is discrete and its image finite, and the outcome set is finite, then convex level sets are not sufficient for the elicitability of $\Gamma$ \citep[Pg. 12]{lambert_elicitation_2019}.\footnote{Instead the level sets have to form a Voronoi diagram \citep[Theorem 1]{lambert_elicitation_2019}.} Hence, even if we require $\Gamma$ to be sensible to calibration, Definition~\ref{def: Gamma-calibration} extends to discrete properties with convex level sets, which are potentially neither identifiable nor elicitable. This justifies the definition of $\Gamma$-calibration with respect to any $\Gamma$ with convex level sets. Furthermore, even if a property $\Gamma$ is not sensible to calibration, a $\Gamma$-predictor could still be $\Gamma$-calibrated. However, in this case optimal individual prediction via the true property predictor and calibration are \emph{not} compatible. Nevertheless, $\Gamma$-calibration is still a fulfillable criterion.

\section{Examples of Properties with Regular Identification Functions}
\label{app: Examples of Properties with Regular Identification Functions}
A priori it is unclear whether there exist interesting, non-trivial, identifiable properties with identification functions fulfilling the assumptions put forward in Definition~\ref{def:propertis of identification function}. For this reason we provide a short list of examples.
\begin{example}
\label{example:properties with nice identification functions}
\begin{description}
    \item[Mean] An easy example for a property with a monotone, locally non-constant, locally Lipschitz identification function is the mean, with identification function $V(y, \gamma) \coloneqq (y - \gamma)$, because
    \begin{align*}
        \mathbb{E}_{Y \sim P}[(Y - \gamma)] = \Gamma(P) - \gamma.
    \end{align*}
    \item[Quantiles] One identification function for a $\tau$-quantile is $V_\tau(y, \gamma)\coloneqq (1-\tau) \llbracket y < \gamma \rrbracket - \tau \llbracket y > \gamma\rrbracket$. For a loss function, which elicits the $\tau$-quantile see e.g., \citep{rockafellar2002conditional}
    Assuming that $P$ is atomless,
    \begin{align*}
        V(P, \gamma) &= \mathbb{E}_{Y\sim P}[(1-\tau) \llbracket Y < \gamma \rrbracket - \tau \llbracket Y > \gamma\rrbracket]\\
        &= (1-\tau)\mathbb{E}_{Y\sim P}[\llbracket Y < \gamma\rrbracket] - \tau  \mathbb{E}_{Y\sim P}[\llbracket Y > \gamma\rrbracket]\\
        &= (1-\tau)P(Y < \gamma) - \tau  (1 - P(Y \le \gamma))\\
        &= (1-\tau)P(Y < \gamma) - \tau  (1 - P(Y < \gamma) - P(Y = \gamma))\\
        &= (1-\tau)P(Y < \gamma) - \tau  (1 - P(Y < \gamma))\\
        &= P(Y < \gamma) - \tau.
    \end{align*}
    Hence, the identification function is monotone. For local Lipschitzness we have to assume that the cumulative density function of $P$ is Lipschitz with parameter $L$ around the $\tau$-quantile of $P$. For instance, this is true for $\beta$-distributions. Then, we get
    \begin{align*}
        |P(Y < \gamma) - \tau| &= |P(Y < \gamma) - P(Y < \gamma^*)| &\le L|\gamma - \gamma^*|,
    \end{align*}
    where $\gamma^*$ is the $\tau$-quantile. Finally, if the cumulative density function induced through $P$ is locally non-constant with parameter $N > 0$ around the $\tau$-quantile of $P$, then via the same argument it follows that the identification function is locally non-constant (\eg, $\beta$-distribution). In summary, the identification function $V_\tau$ for the $\tau$-quantile inherits the continuity properties from the cumulative density function of the probability distribution. Hence, the assumption is fulfilled when restricting the space of possible conditional probability distributions defined through the regular data distribution $D$.

    \item[Ratios of Expectations] Let $g,h \colon \Y \rightarrow \reals$ be measurable functions and $N < h(y) < M$ for all $y \in \Y$. The ratio of expectations $\Gamma(P) = \frac{\E_P g(Y)}{\E_P h(Y)}$ is identified by,
    \begin{align*}
        V(y, \gamma) = h(y) \gamma - g(y),
    \end{align*}
    because
    \begin{align*}
        V(P, \gamma) = \E_P[h(Y)] \gamma -\E_P[g(Y)] = 0
    \end{align*}
    if and only if,
    \begin{align*}
        \frac{\E_P[g(Y)]}{\E_P[h(Y)]} = \gamma.
    \end{align*}
    Furthermore, $V$ is oriented, locally Lipschitz with $M$,
    \begin{align*}
         |V(P, \gamma)| &= |\E_P[h(Y)] \gamma -\E_P[g(Y)]|\\
          &= \left|\E_P[h(Y)] \left( \gamma -\frac{\E_P[g(Y)]}{\E_P[h(Y)]} \right)\right|\\
          &= \left|\E_P[h(Y)] \right|\left| \left( \gamma -\frac{\E_P[g(Y)]}{\E_P[h(Y)]} \right)\right|\\
          &\le M \left| \gamma -\Gamma(P)\right|.
    \end{align*}
    and locally non-constant with $N$,
    \begin{align*}
         |V(P, \gamma)| &= |\E_P[h(Y)] \gamma -\E_P[g(Y)]|\\
          &= \left|\E_P[h(Y)] \left( \gamma -\frac{\E_P[g(Y)]}{\E_P[h(Y)]} \right)\right|\\
          &= \left|\E_P[h(Y)] \right|\left| \left( \gamma -\frac{\E_P[g(Y)]}{\E_P[h(Y)]} \right)\right|\\
          &\ge N \left| \gamma -\Gamma(P)\right|.
    \end{align*}

    \item[Mean and Variance]
    The identification function of the mean $M$ is $V(y, \gamma) \coloneqq (y - \gamma)$. On the $v$-level set of the mean the variance $\mathbb{V}$ can be identified with
    $V_v(y, \gamma)\coloneqq y^2 - v^2 - \gamma$. Note that for $P \in M^{-1}(v)$,
    \begin{align*}
        V_v(P, \gamma) = \mathbb{E}_{Y \sim P}[Y^2 - v^2 - \gamma] = \mathbb{V}(P) - \gamma.
    \end{align*}

    \item[Quantile and CVar]
    The identification function of a $\tau$-quantile $Q_\tau$ is given above. The $\operatorname{CVar}_{\tau}$ is identified on the $v$-level set of the $\tau$-quantile with $V_v(y, \gamma) = v + \frac{1}{1-\tau} \max(0, y - v) - \gamma$. Note that for $P \in Q_\tau^{-1}(v)$
    \begin{align*}
        V_v(P, \gamma) &= \mathbb{E}_{Y \sim P}\left[v + \frac{1}{1-\tau} \max(0, Y - v)\right]\\
        &= v + \frac{1}{1-\tau} \mathbb{E}_{Y \sim P}[\max(0, Y - v)] -\gamma\\
        &= \operatorname{CVar}_{\tau}(P) - \gamma,
    \end{align*}
    The second last line follows from \citep[Theorem 10]{rockafellar2002conditional}.
\end{description}
\end{example}

\section{Calibration-Swap Regret Bridge under Groups}
\label{appendix:Calibration-Swap-Regret Bridge under Groups}
We call $c \colon \X \rightarrow \{ 0,1\}$ a \emph{group} if $c$ is measurable and $D(c(X) = 1) > 0$. The set $\C$ denotes a collection of groups.
\begin{definition}[$\C$-Robust $\Gamma$-Swap-Learner]
\label{def:robsut swap learner}
    Let $\C$ be a collection of groups and $\Gamma$ a real-valued, identifiable property with oriented, locally Lipschitz ($M$) and locally non-constant ($N$) identification function $V$. Let $\ell$ be a loss function induced through $V$ following Equation~\eqref{eq:loss function induced by identification function}. Let $D$ be a regular data distribution.
    A $\Gamma$-predictor $f$ is a $\beta$-approximate $\C$-robust $\Gamma$-swap-learner on $D$, iff, for all $c \in \C$,
    \begin{align*}
        \mathbb{E}_{D}\left[\ell(Y, f(X)) |c(X) = 1\right] \le \min_{h \in \mathcal{M}_f(\X, \R)} \mathbb{E}_{D}\left[\ell(Y, h(X))|c(X) = 1\right] + \frac{\beta}{\mathbb{E}_{D}\left[ c(X)\right]},
    \end{align*}
    where $\mathcal{M}_f(\X, \R)$ is the set of all measurable functions from $(\X, \sigma(f))$ to $(\R, \mathcal{B}(\R))$, where $\sigma(f)$ is the $\sigma$-algebra induced by the $\Gamma$-predictor $f$ and $\mathcal{B}(\R)$ the Borel-$\sigma$-algebra on $\R$.
\end{definition}
The definition of a $\C$-robust swap-learner actually is inspired by minimax regret in distributional robust optimization. Formalized as an optimization problem the regret here is a special case of a minimax regret optimization problem as in \citep{agarwal2022minimax}. The class of functions $\F$ in their Equation (1) is $\mathcal{M}_f(\X, \R)$ in our case and the set of probability distributions $\P$ in their Equation (1) corresponds to the conditional probabilities on $\X \times \Y$ induced by conditioning on the groups $\C$ in our case. The predictor $f \in \mathcal{M}_f(\X, \R)$.
\begin{proposition}[$l^2$-Multicalibration implies Robust Swap-Learner]
\label{prop:l^2-Multicalibration implies Robust Swap-Learner}
 Let $\C$ be a collection of groups and $\Gamma$ a real-valued, identifiable property with oriented, locally Lipschitz ($M$) and locally non-constant ($N$) identification function $V$. Let $\ell$ be a loss function induced through $V$ following Equation~\eqref{eq:loss function induced by identification function}. Let $D$ be a regular data distribution.
    Let $f$ be $\alpha$-approximately $(\C, \nu)$-multicalibrated following Definition 15 in \citep{noarov_statistical_2023}. Then, $f$ is a $\frac{M}{2} \alpha$-approximate $\C$-robust $\Gamma$-swap-learner.
\end{proposition}
\begin{proof}
    \begin{align*}
        \mathbf{\alpha} &\ge \sup_{c \in C} \mathbb{E}_{X \sim D_{X}}[c(X)] \cdot \mathbb{E}_{\gamma \sim D_{f(X)|c(X) = 1}}[|\Gamma(D_{Y | \{ x \colon f(x) = \gamma, c(x) = 1\}}) - \gamma|^2]\\
        &\overset{P\ref{prop:Approximate Calibration and Low Swap Regret}}{\ge} \frac{2}{M}\sup_{c \in C} \mathbb{E}_{X \sim D_{X}}[c(X)] \cdot\\
        &\quad \left(\mathbb{E}_{\gamma \sim D_{f(X)|c(X) = 1}}[\mathbb{E}_{D}\left[ \ell(Y, \gamma)| f(X) = \gamma , c(X) = 1 \right] - \min_{\hat{\gamma} \in \im \Gamma}\mathbb{E}_{D}\left[ \ell(Y, \hat{\gamma}) | f(X) = \gamma , c(X) = 1 \right] \right)\\
        &= \frac{2}{M}\sup_{c \in C} \mathbb{E}_{D}\left[ c(X)\ell(Y, f(X)) \right]\\
        & \quad - \mathbb{E}_{X \sim D_{X}}[c(X)] \cdot \mathbb{E}_{\gamma \sim D_{f(X)|c(X) = 1}}[\min_{\hat{\gamma} \in \im \Gamma}\mathbb{E}_{D}\left[ \ell(Y, \hat{\gamma}) | f(X) = \gamma , c(X) = 1 \right]]\\
        &= \frac{2}{M}\sup_{c \in C} \mathbb{E}_{D}\left[ c(X)\ell(Y, f(X)) \right] - \min_{h \in \mathcal{M}_f(\X, \R)} \mathbb{E}_{D}\left[ c(X)\ell(Y, h(X))\right]
    \end{align*}
    
    The final step is a consequence of conditional expectations and the following argument. For every $c \in \C$, applying \citep[Theorem 14.60]{rockafellar2009variational} with $T = \{ x\in \X\colon c(x) = 1\}$ and $\A = \B(\X) \cap T$ gives,
    \begin{align*}
        \mathbb{E}_{\gamma \sim D_{f(X)|c(X) = 1}}&\left[\min_{\hat{\gamma} \in \im \Gamma}\mathbb{E}_{D}\left[ \ell(Y, \hat{\gamma}) | f(X) = \gamma , c(X) = 1 \right]\right]\\
        &= \mathbb{E}_{\gamma \sim D_{f(X)|c(X) = 1}}\left[\min_{\alpha \in \reals} \mathbb{E}_{D}\left[  \ell(Y, \alpha) | f(X) = \gamma , c(X) = 1 \right]\right]\\
        &= \min_{h \in \mathcal{M}_f(\X, \R)} \mathbb{E}_{\gamma \sim D_{f(X)|c(X) = 1}}\left[\mathbb{E}_{D}\left[  \ell(Y, h(X)) | f(X) = \gamma , c(X) = 1 \right]\right]\\
        &= \min_{h \in \mathcal{M}_f(\X, \R)} \mathbb{E}_{D}\left[  \ell(Y, h(X)) | c(X) = 1 \right].
    \end{align*}
\end{proof}

\begin{proposition}[Robust Swap-Learner implies $l^2$-Multicalibration]
\label{prop:Robust Swap-Learner implies l^2-Multicalibration}
 Let $\C$ be a collection of groups and $\Gamma$ a real-valued, identifiable property with oriented, locally Lipschitz ($M$) and locally non-constant ($N$) identification function $V$. Let $\ell$ be a loss function induced through $V$ following Equation~\eqref{eq:loss function induced by identification function}. Let $D$ be a regular data distribution.
    Let $f$ be a $\beta$-approximate $\C$-robust $\Gamma$-swap-learner. Then, $f$ is $\frac{2L^2}{N} \beta$-approximately $(\C, \nu)$-multicalibrated following Definition 15 in \citep{noarov_statistical_2023}.
\end{proposition}
\begin{proof}
    \begin{align*}
        \beta &\ge \sup_{c \in C} \mathbb{E}_{D}\left[ c(X)\ell(Y, f(X)) \right] - \min_{h \in \mathcal{M}_f(\X, \R)} \mathbb{E}_{D}\left[ c(X)\ell(Y, h(X))\right]\\
        &=\sup_{c \in C} \mathbb{E}_{D}\left[ c(X)\ell(Y, f(X)) \right]\\
        & \quad - \mathbb{E}_{X \sim D_{X}}[c(X)] \cdot \mathbb{E}_{\gamma \sim D_{f(X)|c(X) = 1}}[\min_{\hat{\gamma} \in \im \Gamma}\mathbb{E}_{D}\left[ \ell(Y, \hat{\gamma}) | f(X) = \gamma , c(X) = 1 \right]]\\
        &=\sup_{c \in C} \mathbb{E}_{X \sim D_{X}}[c(X)] \cdot\\
        &\quad \left(\mathbb{E}_{\gamma \sim D_{f(X)|c(X) = 1}}[\mathbb{E}_{D}\left[ \ell(Y, \gamma)| f(X) = \gamma , c(X) = 1 \right] - \min_{\hat{\gamma} \in \im \Gamma}\mathbb{E}_{D}\left[ \ell(Y, \hat{\gamma}) | f(X) = \gamma , c(X) = 1 \right] \right)\\
        &\overset{P\ref{prop:Approximate Calibration and Low Swap Regret}}{\ge}  \frac{N}{2} \sup_{c \in C} \mathbb{E}_{X \sim D_{X}}[c(X)] \cdot \mathbb{E}_{\gamma \sim D_{f(X)|c(X) = 1}}\left[|\Gamma(D_{Y | \{ x \colon f(x) = \gamma, c(x) = 1\}}) - \gamma|^2\right],
    \end{align*}
    reversing the steps of the proof of Proposition~\ref{prop:l^2-Multicalibration implies Robust Swap-Learner}.
\end{proof}

\section{Self-Realization Does Not Imply Equal to Usefulness for All}
\label{appendix:Self-Realization Does Not Imply Equal to Usefulness for All}
In the following appendix we provide a short, concise example to show that calibration does not imply that the low swap regret guarantee is equally useful for all. Note that conceptually the statements could seem strange to the reader. The predictor in our case has access to the nature's outcome. This surely is unrealistic. Nevertheless, it does not undermine the argument that a perfectly calibrated predictor exists which is of different use to different downstream decision makers.

For the example we reuse \emph{simple loss functions} (Definition~\ref{def:simple loss function}. They describe binary-valued elicitable properties on binary outcome sets.
The optimal achieveable risk, Bayes risk, for a distribution $P \in \Delta(\Y)$ for $\Y = \{ 0,1\}$ is given by,
\begin{align*}
    \mathcal{B}_q(P) = \mathbb{E}_{Y \sim P}[\ell_q(Y, \Phi_q(P))].
\end{align*}
\begin{lemma}[Calibration Does not Guarantee Cost Parity for Arbitrary Downstream Decision Makers with Equal Bayes Risk]
\label{lemma:Calibration Does not Guarantee Cost Parity for Arbitrary Downstream Decision Makers with Equal Bayes Risk}
    Let $\X$ be finite (with size $T$), $\Y = \{ 0,1\}$. Let $\ell_c$ and $\ell_d$ be simple losses with $c,d \in (0,1)$ and $d < c$. For every choice of $c,d$ there exists a regular data distribution $D \in \Delta(\X \times \Y)$ with the $\X$-marginal being the uniform distribution and $Q \in \Delta(\Y)$ being a stationary $Y|X$-conditional such that
    \begin{enumerate}[(i)]
        \item the Bayes risks are equal $\mathcal{B}_c(Q) = \mathcal{B}_d(Q)$.
        \item and there exists a calibrated predictor $p\colon \X \to \Delta(\Y)$ such that,
        \begin{align*}
            |\mathbb{E}_{(X,Y)\sim D}[\ell_c(Y, \Phi_c(p(X)))] - \mathbb{E}_{(X,Y)\sim D}[\ell_d(Y, \Phi_d(X))]| \ge C,
        \end{align*}
        for some $C \in \reals$.
    \end{enumerate}
\end{lemma}
\begin{proof}
    We first define the stationary conditional distribution $Q$ on the binary outcome set via $q \in [0,1]$. Let $q = \frac{1}{\frac{1-c}{d} + 1}$. Observe $d < \frac{1}{\frac{1-c}{d} + 1} < c$.

    It remains to check the Bayes risk condition,
    \begin{align*}
        \mathcal{B}_c(Q) &= (1-c) q\\
        &= \frac{1-c}{\frac{1-c}{d} + 1} \\
        &= d \left(1-\frac{1}{\frac{1-c}{d} + 1}\right)\\
        &= d (1-q) = \mathcal{B}_d(Q).
    \end{align*}

    Now, let us define the predictor $p$.

    We define the following prediction sequence for every $\omega \in \Omega$: if $y(\omega) = 0$, then $p(X(\omega)) = f$ for some $f \in [0,1]$ such that $d < f < q$. If $y(\omega) = 1$, then with probability
     \begin{align*}
         x \coloneqq \frac{(1-q)f}{q(1-f)},
     \end{align*}
    the prediction is $p(X(\omega)) = f$, otherwise $p(X(\omega)) = 1$. Since $f < q$ and $d,c \in (0,1)$,
    \begin{align*}
        0 < x = \frac{(1-q)f}{q(1-f))} < 1.
    \end{align*}

    For the second, note that all predictions $p(X(\omega)) = 1$ are calibrated by definition. For the predictions $p(X(\omega)) = f$ we have
    \begin{align*}
        \mathbb{E}_{Y \sim Q}[Y| p(X) = f] &= \frac{q x}{(1-q) + qx}\\
        &= \frac{q \frac{(1-q)f}{g(1-f)}}{(1-q) + q\frac{(1-q)f}{q(1-f)}}\\
        &= \frac{\frac{(1-q)f}{1-f}}{\frac{(1-q)(1-f)}{1-f} + \frac{(1-q)f}{1-f}}\\
        &= \frac{\frac{(1-q)f}{1-f}}{\frac{(1-q)}{1-f}}\\
        &= f.
    \end{align*}

    So, we can finally show the statement
    \begin{align*}
        &|\mathbb{E}_{(X,Y)\sim D}[\ell_c(Y, \Phi_c(p(X)))] - \mathbb{E}_{(X,Y)\sim D}[\ell_d(Y, \Phi_d(X))]|\\
        &= \left|\mathbb{E}_{X \sim D_X}[q x (1-c) - (1-q) d]\right|\\
        &= (1-q) \left|\frac{f}{1-f} (1-c) - d\right|\\
        &\ge C,
    \end{align*}
    because $\frac{f}{1-f} (1-c) - d \neq 0$ if $f < q$.
\end{proof}
In other words, the predictions, even though they are calibrated, hence they fulfill some notion of omniprediction (\cf Section~\ref{sec:self-realization-inherited}), do not guarantee that the predictions are equally useful for every decision maker. This holds despite the forecasting task itself is equally difficult as measured by the Bayes risk.

\section{Proofs and Lemmas}

\begin{lemma}
\label{lemma:example for outcome lipschitzness}
    Let $\Y$ be finite and $\Gamma$ be the full distribution property, i.e., $\R = \Delta(\Y)$. If $\ell \colon \Y \times \R' \to \reals$ is a bounded loss function, then $\ell$ is locally outcome Lipschitz with respect to $\Gamma$.
\end{lemma}
\begin{proof}
    Let $\Phi$ be the property elicited through $\ell$. For all $v \in \im \Phi$, $\gamma \in \im \Gamma$, $P \in \Gamma^{-1}(\gamma)$ and $P' \in \P$,
    \begin{align*}
        &| \ell(P, v) - \ell(P', v) |\\
        &= | \sum_{y \in \Y}(P(Y = y) - P'(Y = y)) \ell(y, v)|\\
        &\le C | \sum_{y \in \Y}(P(Y = y) - P'(Y = y))|\\
        &\le C \sum_{y \in \Y}|(P(Y = y) - P'(Y = y))|\\
        &\le C m(\Gamma(P'), \gamma),
    \end{align*}
    because $m$ is the total variation distance.
\end{proof}

\begin{lemma}[Identification Function Defines Consistent Loss Functions]
\label{lemma:Identification Function Defines Consistent Loss Function}
    Let $\Gamma \colon \P \rightarrow \R$ with $\R \subseteq \reals$ be an identifiable property with oriented, bounded identification function $V \colon \Y \times \R \rightarrow \reals$ which is measurable in both its inputs with respect to the standard Borel-$\sigma$-algebra on the sets $\Y, \R$ and $\reals$.
    Then, $\Gamma$ is elicitable with $\P$-consistent scoring function $\ell \colon \Y \times \R \rightarrow \reals$,
    \begin{align}
    \label{eq:loss function induced by identification function - proof}
        \ell(y, \gamma) \coloneqq \int_{\gamma_0}^\gamma V(y,r) dr + \kappa(y),
    \end{align}
    for some $\gamma_0 \in \im \Gamma$ and $\kappa \colon \Y \rightarrow \reals$ having a finite expected value with respect to all $P \in \P$.

    \begin{enumerate}[(a)]
        \item If furthermore, $V$ is locally Lipschitz with parameter $M$, then $\ell$ is locally Hölder-smooth with parameters $(\frac{M}{2}, 2)$, \ie, 
        \begin{align}
            \label{eq:locally hoeldersmooth loss function}
            |\ell(P, \gamma) - \ell(P, \Gamma(P))| \le \frac{M}{2} \left( \gamma - \Gamma(P)\right)^2.
        \end{align}
        \item If furthermore, $V$ is locally non-constant with parameter $N$, then $\ell$ is locally anti Hölder-smooth with parameters $(\frac{N}{2}, 2)$, \ie, 
        \begin{align}
            \label{eq:locally anti hoeldersmooth loss function}
            \frac{N}{2} \left( \gamma - \Gamma(P)\right)^2 \le |\ell(P, \gamma) - \ell(P, \Gamma(P))|.
        \end{align}
    \end{enumerate}
\end{lemma}
\begin{proof}
    The first statement has been proved in \citep{steinwart2014elicitation}. We provide a reiteration of the argument for the sake of completeness, and because we will reuse some equations for the additional statements.

    For any $P \in \P$, we want to show that $\mathbb{E}_{Y \sim P}[\ell(Y, \gamma)] > \mathbb{E}_{Y \sim P}[\ell(Y, \Gamma(P))]$ for all $\gamma \in A$, $\gamma \neq \Gamma(P)$. Let us first consider the case that $\gamma > \Gamma(P)$. Hence,
    \begin{align*}
        &\mathbb{E}_{Y \sim P}[\ell(Y, \gamma)] - \mathbb{E}_{Y \sim P}[\ell(Y, \Gamma(P))]\\
        &= \mathbb{E}_{Y \sim P}[\ell(Y, \gamma) - \ell(Y, \Gamma(P))]\\
        &= \mathbb{E}_{Y \sim P}\left[\int_{\Gamma(P)}^\gamma V(y,r) dr\right].
    \end{align*}
    Since $V$ is measurable and bounded in both variables we can apply Fubini's theorem, this gives
    \begin{align}
        \label{eq:integral swap for identification function}
        \mathbb{E}_{Y \sim P}\left[\int_{\Gamma(P)}^\gamma V(y,r) dr\right] = \int_{\Gamma(P)}^\gamma \mathbb{E}_{Y \sim P}[V(y,r)] dr,
    \end{align}
    finally, since $V$ is oriented,
    \begin{align*}
        \int_{\Gamma(P)}^\gamma \mathbb{E}_{Y \sim P}[V(y,r)] dr = \int_{\Gamma(P)}^\gamma V(P,r)dr > 0.
    \end{align*}
    The case $\gamma < \Gamma(P)$ follows analogously.

    \begin{enumerate}[(a)]
        \item Let $\gamma > \Gamma(P)$, then,
        \begin{align*}
            |\ell(P, \gamma) - \ell(P, \Gamma(P))| &= |\int_{\Gamma(P)}^\gamma V(P,r)dr|\\
            &= \int_{\Gamma(P)}^\gamma |V(P,r)| dr\\
            &\le \int_{\Gamma(P)}^\gamma M |r - \Gamma(P)| dr\\
            &= \frac{M}{2}(\gamma - \Gamma(P))^2.
        \end{align*}
        The analogous computation with swapped signs holds for $\gamma < \Gamma(P)$.
        \item Let $\gamma > \Gamma(P)$, then,
        \begin{align*}
            |\ell(P, \gamma) - \ell(P, \Gamma(P))| &= |\int_{\Gamma(P)}^\gamma V(P,r)dr|\\
            &= \int_{\Gamma(P)}^\gamma |V(P,r)| dr\\
            &\ge \int_{\Gamma(P)}^\gamma N|r - \Gamma(P)| dr\\
            &= \frac{N}{2}(\gamma - \Gamma(P))^2.
        \end{align*}
        The analogous computation with swapped signs holds for $\gamma < \Gamma(P)$.
    \end{enumerate}
\end{proof}

\begin{lemma}[When Decision Calibration is Equivalent to Calibration on Decision Sets]
\label{lemma:When Decision Calibration is Equivalent to Calibration on Decision Sets}
    Let $\Y = \{ 0,1\}$ and $\X$ arbitrary and $D$ be a regular data distribution on $\X$ and $\Y$. Suppose that $f \colon \X \rightarrow [0,1]$ is a mean-predictor, i.e., it predicts the probability of $Y = 1$. Furthermore, $f$ matches the average, i.e., $\mathbb{E}_{(X,Y) \sim D}[Y -f(X)] = 0$. Let $\ell_q$ be a simple loss function with parameter $q \in [0,1]$. In addition, $P_{D}(f(X) \le q) > 0$.\\
    Then, $f$ is calibrated on the decision sets of $\ell_q$, i.e., $\mathbb{E}_{(X,Y) \sim D}[Y-f(X)| f(X) \le q] = 0$ and $\mathbb{E}_{(X,Y) \sim D}[Y-f(X)| f(X) > q] = 0$\footnote{If $f$ matches the average, the first assumption implies the second.}, if and only if, $f$ is decision calibrated with respect to the loss function $\ell_q$.
\end{lemma}
\begin{proof}
    Let us rewrite the decision calibration term,
    \begin{align*}
        &\mathbb{E}_{X \sim D_X}[\mathbb{E}_{Y \sim D_{Y|X=x}}[\ell_q(Y, f(x))] - \mathbb{E}_{Y \sim f(X)}[\ell_q(Y, f(x))]]\\
        &= \mathbb{E}_{X \sim D_X}[ P(Y = 0|X=X) q \llbracket f(X) > q\rrbracket  + P(Y = 1|X=X) (1-q) \llbracket f(X) \le q\rrbracket  \\
        &-  (1-f(X)) q \llbracket f(X) > q\rrbracket - f(X) (1-q) \llbracket f(X) \le q\rrbracket ]\\
        &= \mathbb{E}_{X \sim D_X}[ (P(Y=1|X=x) - f(X))  q \llbracket f(X) > q\rrbracket \\
        &+ (f(X)- P(Y=1|X=x)) (1-q) \llbracket f(X) \le q\rrbracket ]\\
        &= q\mathbb{E}_{X \sim D_X}[ (P(Y=1|X=x) - f(X)) \llbracket f(X) > q\rrbracket ]\\
        &+ (1-q) \mathbb{E}_{X \sim D_X}[(f(X)- P(Y=1|X=x)) \llbracket f(X) \le q\rrbracket ]\\
        &\eqqcolon L(q).
    \end{align*}

    Note that the following two equations hold,
    \begin{align*}
        &\mathbb{E}_{X \sim D_X}[ (P(Y=1|X=x) - f(X)) \llbracket f(X) > q\rrbracket]\\
        &= \mathbb{E}_{X \sim D_X}[ (P(Y=1|X=x) - f(X))|f(X) > q] P_{D}(f(X) > q)\\
        &= \mathbb{E}_{(X,Y) \sim D}[ (Y- f(X))|f(X) > q] P_{D}(f(X) > q),
    \end{align*}
    and,
    \begin{align*}
        &\mathbb{E}_{X \sim D_X}[ (f(X) - P(Y=1|X=x)) \llbracket f(X) \le q\rrbracket]\\
        &= \mathbb{E}_{X \sim D_X}[ (f(X)- P(Y=1|X=x))|f(X) \le q] P_{D}(f(X) \le q)\\
        &= \mathbb{E}_{(X,Y) \sim D}[ (f(X)- Y)|f(X) \le q] P_{D}(f(X) \le q).
    \end{align*}

    With these reformulations at hand, it is easy to see that calibration on decision sets of $\ell_q$ implies decision calibration with respect to the loss function $\ell_q$, because,
    \begin{align*}
        L(q) &= q\mathbb{E}_{X \sim D_X}[ (P(Y=1|X=x) - f(X))|f(X) > q] P_{D}(f(X) > q)\\
        &+ (1-q) \mathbb{E}_{X \sim D_X}[ (f(X) - P(Y=1|X=x))|f(X) \le q] P_{D}(f(X) \le q)\\
        &= q 0  P_{D}(f(X) > q) + (1-q) 0 P_{D}(f(X) \le q) = 0.
    \end{align*}

    The reverse implication requires a bit more reasoning. We argue via contraposition. Let us assume that $\mathbb{E}_{(X,Y) \sim D}[Y-f(X)| f(X) \le q] \neq 0$.
    Note that,
    \begin{align*}
        &\mathbb{E}_{(X,Y) \sim D}[Y -f(X)]\\
        &= \mathbb{E}_{X \sim D_X}[ (P(Y=1|X=x) - f(X))|f(X) > q] P_{D}(f(X) > q) \\
        &+ \mathbb{E}_{X \sim D_X}[ (P(Y=1|X=x) - f(X))|f(X) \le q] P_{D}(f(X) \le q)\\
        &= 0.
    \end{align*}
    Hence,
    \begin{align*}
        &\mathbb{E}_{X \sim D_X}[ (P(Y=1|X=x) - f(X))|f(X) > q] P_{D}(f(X) > q)\\
        &= - \mathbb{E}_{X \sim D_X}[ (P(Y=1|X=x) - f(X))|f(X) \le q] P_{D}(f(X) \le q),
    \end{align*}
    respectively,
    \begin{align*}
        &\mathbb{E}_{X \sim D_X}[ (f(X) - P(Y=1|X=x)) \llbracket f(X) > q\rrbracket]\\
        &= \mathbb{E}_{X \sim D_X}[ (f(X)-P(Y=1|X=x)) \llbracket f(X) \le q\rrbracket]. 
    \end{align*}
    Now,
    \begin{align*}
        L(q)
        &= \mathbb{E}_{X \sim D_X}[ (P(Y=1|X=x) - f(X))|f(X) \le q] P_{D}(f(X) \le q)\\
        &\neq 0,
    \end{align*}
    because $P_{D}(f(X) \le q) > 0$ by assumption.
\end{proof}

\begin{lemma}[Matching Averages by Non-Extremal Predictions I]
\label{lemma:matching averages by non-extremal predictions I}
    Let $\Y = \{ 0,1\}$ and $\X$ arbitrary and $D$ be a regular data distribution on $\X$ and $\Y$. Suppose that $f \colon \X \rightarrow [0,1]$ is a mean-predictor, i.e., it predicts the probability of $Y = 1$. Furthermore, we assume there exists $f_{\inf} \coloneqq \inf \operatorname{im} f > 0$.
    Then, $f$ matches the average, i.e., $\mathbb{E}_{(X,Y) \sim D}[Y -f(X)] = 0$, if and only if $f$ is decision calibrated with respect to the loss function $\ell_{\frac{f_{\inf}}{2}}$.
\end{lemma}
\begin{proof}
    Note that,
    \begin{align*}
        \mathbb{E}_{(X,Y) \sim D}[Y -f(X)] &= \mathbb{E}_{(X,Y) \sim D}\left[(P(Y=1|X=x) -f(X) )\mathbf{1}\left(f(X) > \frac{f_{\inf}}{2}\right)\right].
    \end{align*}
    and
    \begin{align*}
        L\left(\frac{f_{\inf}}{2}\right) &= \frac{f_{\inf}}{2}\mathbb{E}_{X \sim D_X}\left[ (P(Y=1|X=x) - f(X)) \mathbf{1}\left[f(X) > \frac{f_{\inf}}{2}\right]\right]\\
        &+ \left(1-\frac{f_{\inf}}{2}\right) \mathbb{E}_{X \sim D_X}\left[(f(X)- P(Y=1|X=x)) \mathbf{1}\left[f(X) \le \frac{f_{\inf}}{2}\right] \right]\\
        &= \frac{f_{\inf}}{2} \mathbb{E}_{(X,Y) \sim D}[Y -f(X)].
    \end{align*}
    Since $\frac{f_{\inf}}{2} > 0$, $L(f_{\inf}) = 0$ if and only if $\mathbb{E}_{(X,Y) \sim D}[Y -f(X)] = 0$.
\end{proof}
\begin{lemma}[Matching Averages by Non-Extremal Predictions II]
\label{lemma:matching averages by non-extremal predictions II}
    Let $\Y = \{ 0,1\}$ and $\X$ arbitrary and $D$ be a regular data distribution on $\X$ and $\Y$. Suppose that $f \colon \X \rightarrow [0,1]$ is a mean-predictor, i.e., it predicts the probability of $Y = 1$. Furthermore, we assume $f_{\inf} \coloneqq \inf \operatorname{im} f = 0$, and $|\im f| < \infty$, hence there exists $v \in [0,1]$ such that $v > 0$ but $v < v'$ for all $v' \in \im f \setminus \{0\}$.
    Then, $f$ matches the average, i.e., $\mathbb{E}_{(X,Y) \sim D}[Y -f(X)] = 0$, if and only if $f$ is decision calibrated with respect to the loss functions $\ell_{v}$ and $\ell_0$.
\end{lemma}
\begin{proof}
    It holds $L\left(0\right) = 0$ and $L\left(v\right) = 0$. In detail,
    \begin{align*}
        L\left(0\right) &= \mathbb{E}_{X \sim D_X}\left[(f(X)- P(Y=1|X=x)) \mathbf{1}\left[f(X) = 0\right] \right] = 0.
    \end{align*}

    \begin{align*}
        L\left(v\right) &= v \mathbb{E}_{X \sim D_X}\left[ (P(Y=1|X=x) - f(X)) \mathbf{1}\left[f(X) > v\right]\right]\\
        &+ \left(1-v\right) \mathbb{E}_{X \sim D_X}\left[(f(X)- P(Y=1|X=x)) \mathbf{1}\left[f(X) = 0\right] \right]\\
        &= v \mathbb{E}_{X \sim D_X}\left[ (P(Y=1|X=x) - f(X)) \mathbf{1}\left[f(X) > v\right]\right] = 0.
    \end{align*}
    Hence, since $v > 0$, $\mathbb{E}_{X \sim D_X}\left[ (P(Y=1|X=x) - f(X)) \mathbf{1}\left[f(X) > v\right]\right] = 0$. This gives,
    \begin{align*}
        &\mathbb{E}_{X \sim D_X}\left[ (P(Y=1|X=x) - f(X)) \mathbf{1}\left[f(X) > v\right]\right]\\
        &\quad + \mathbb{E}_{X \sim D_X}\left[(f(X)- P(Y=1|X=x)) \mathbf{1}\left[f(X) = 0\right] \right]\\
        &=\mathbb{E}_{(X,Y) \sim D}[Y -f(X)] = 0.
    \end{align*}
\end{proof}

\bibliographystyle{abbrvnat}
\bibliography{refs}

\end{document}